\documentclass[10pt]{article}
\usepackage[margin=1in]{geometry}
\usepackage[T1]{fontenc}
\usepackage{lmodern}
\usepackage{natbib}
\usepackage{amsmath, amssymb, amsthm}
\usepackage{bm, graphicx}
\usepackage{algorithm}
\usepackage{algpseudocode}
\usepackage{booktabs}
\usepackage{longtable}
\usepackage{placeins}
\usepackage{pifont}
\usepackage{multirow}
\usepackage{tikz}
\usetikzlibrary{arrows.meta, positioning, decorations.pathreplacing}
\usepackage[colorlinks=true]{hyperref}
\newtheorem{theorem}{Theorem}
\newtheorem{definition}{Definition}
\newtheorem{proposition}{Proposition}

\newtheorem{corollary}{Corollary}

\newtheorem{assumption}{Assumption}

\newcommand{\R}{\mathbb{R}}

\newcommand{\TV}{\mathrm{TV}}

\begin{document}

\title{Dynamics-Level Watermarking of Flow Matching Models with Random Codes}
\author{Shuchan Wang}
\date{\today}

\maketitle

\begin{abstract}
Ownership marks for generative models are usually placed in generated
content, decoders, or weights. We ask whether the learned transport
dynamics of a flow-matching model can serve as an additional ownership
channel. We embed a keyed multi-bit message in the velocity field as a
secret spatial code modulated by a temporal carrier. We recover the message
by synchronous demodulation of black-box velocity queries, without access
to model weights or an image watermark extractor. We demonstrate the method
on image and audio generators, with measured feature distances close to
those of the corresponding clean controls. We analyze
recovery in terms of code geometry, background variation, and query count,
and bound paired terminal displacement and terminal-distribution drift
under stated regularity assumptions. Modification and transfer experiments show
that the mark persists through moderate compression and transfers to
students trained to match source velocities. Under function-preserving
channel permutations, the velocity verifier retains complete message
recovery, while the weight-based baseline recovers no complete messages.
Keyless carrier subtraction
reduces recovery to near chance, and clean-data fine-tuning reduces recovery.
These results establish the feasibility of dynamics-level ownership
watermarking for non-adversarial auditing of models that expose velocity
queries.
\end{abstract}

\section{Introduction}

Flow matching~\citep{lipmanflow,liuflow} trains continuous normalizing flows by conditional velocity regression. Given samples $x_1\sim
q_{\text{data}}$ from a data distribution and $x_0\sim p_{\text{noise}}$
from a simple base distribution, flow matching learns a time-dependent
velocity field $v_\theta(x,t)$ that transports $p_{\text{noise}}$ toward
$q_{\text{data}}$ through the ODE
\begin{equation}
\frac{dx_t}{dt}=v_\theta(x_t,t),\qquad t\in[0,1].
\end{equation}
The model is trained by regressing $v_\theta$ against a conditional
velocity field defined along interpolating paths
$x_t=(1-t)x_0+tx_1$, with target $U=x_1-x_0$. Integrating the learned ODE
from $x_0\sim p_{\text{noise}}$ then produces a generated sample.

The learned velocity field
$v_\theta:\mathbb R^D\times[0,1]\to\mathbb R^D$ defines the generative dynamics. Training
such a model requires substantial computational resources, making the
resulting velocity field valuable intellectual property in need of
protection against unauthorized use.

\subsection{Problem Setting and Design Goals}

We consider an owner who releases a generative flow while retaining a
secret verification key. An auditor can query the model's velocity
responses without accessing its parameters; a recipient may modify the
released model or train a replacement from its velocities or generated
samples. We therefore evaluate both model modification and transfer under
velocity and sample supervision.

The ownership signal should reside in the model, support keyed multi-bit
identification, and be recoverable from model queries while preserving
generation quality and limiting terminal-sample distortion. Because
verification operates on the represented velocity function rather than
parameter coordinates, it is invariant to function-preserving
reparameterizations, as evaluated by the hidden-channel permutation
control in Section~\ref{sec:functional-invariance}.

Limiting output changes also matters beyond images. Flow models are used
for protein structures and physically constrained
systems~\citep{yim2023frameflow,utkarsh2025pcfm}, where a changed generated
state can affect structural validity or physical constraints. These uses
motivate embedding ownership information in the generating dynamics while
controlling its terminal effect; our experiments evaluate this principle
on image and audio models.

\subsection{Existing Approaches}\label{sec:related}

We categorize existing watermarking methods by where they embed information
and how they verify it. Methods that watermark generated content include
Tree-Ring, which modifies sampling noise and detects the mark through
image inversion~\citep{wen2023tree}, and Stable Signature, which trains a
latent decoder to produce watermarked images~\citep{fernandez2023stable}.
Both verify ownership from generated images. Weight-level methods embed
bits into model parameters and recover them through parameter
access~\citep{uchida2017embedding}. Behavior-level methods use selected
inputs that produce predetermined ownership
responses~\citep{adi2018turning,chen2019blackmarks,fei2025high}.
Recent methods support multi-bit messages through channel
coding~\citep{fei2025high}, while SleeperMark trains triggered
diffusion-model responses to survive fine-tuning~\citep{wang2025sleepermark}.

These methods verify ownership through output statistics, parameter
projections, or trigger-conditioned responses.
Table~\ref{tab:comparison} in Appendix~\ref{app:comparison} summarizes
these verification interfaces. In flow matching, the learned continuous
velocity field provides an additional ownership channel.

\subsection{Watermarking the Learned Dynamics}

We study \emph{dynamics-level watermarking}: encoding keyed messages in
the learned velocity field with controlled terminal distortion. Our
central question is:
\begin{quote}
\emph{Can we embed a verifiable, keyed, multi-bit message directly into
the learned velocity field of a flow-matching model, recover it through
velocity queries, while maintaining generation quality close to the clean
model?}
\end{quote}

Our approach adds a key-dependent spatial signal $Pc_m$, multiplied by a
temporal carrier $\alpha(t)$, to the clean velocity reference. The
construction permits general carriers; the experiments use
$\alpha(t)=\sin(2\pi t)$ as the primary design. The resulting watermark
is a property of the learned velocity function rather than a dedicated
decoder, output postprocessor, or detachable watermark module. The verifier applies spatial projection and synchronous demodulation to
velocity queries, then decodes the averaged statistic
(Figure~\ref{fig:overview}). We analyze how code geometry, background
velocity variation, and query count affect recovery. To characterize the
signal's effect on generation, we bound paired endpoint displacement and
apply transport stability to the output distribution. A balanced sinusoidal reference provides an exact endpoint identity;
experiments compare its learned implementation with additive sine and
measure total marked-versus-clean output distortion.

The carrier must be learned through velocity regression and resolved by
finite query averaging. Section~\ref{sec:carrier-design} compares the
single sinusoid with energy-matched frequency- and time-division designs
while holding the spatial code and optimization protocol fixed.

\subsection{Our Contributions}

The paper makes three contributions: \textit{i)} a training objective and
query procedure using a temporal carrier for embedding and
recovering keyed multi-bit messages in flow-matching velocity fields; \textit{ii)} an analysis of message recovery
and output changes, connecting code geometry and query count to decoding,
and velocity perturbations to final samples and output distributions;
and \textit{iii)} image and audio experiments demonstrating recovery and
measuring generation quality under model modification and student training,
with controls using random student initialization and hidden-channel
permutations that leave the model function unchanged. We compare against weight,
trigger, and time-independent baselines under the same modifications.

\begin{figure}[!htb]
\centering
\resizebox{\textwidth}{!}{%
\begin{tikzpicture}[
    node distance=0.6cm,
    box/.style={rectangle, draw, rounded corners, minimum width=2.3cm, minimum height=0.7cm, align=center, font=\footnotesize},
    smallbox/.style={rectangle, draw, rounded corners, minimum width=1.7cm, minimum height=0.7cm, align=center, font=\footnotesize},
    arrow/.style={->, >=stealth, thick},
]

\node[font=\small\bfseries, anchor=west] (trainlabel) at (-5.5, 0) {Training (Watermark Embedding)};

\node[box, below left=0.5cm and -2.5cm of trainlabel] (key)
{Message $m$\\Key $K=(P,C)$};
\node[box, right=0.6cm of key] (signal) {Watermark signal\\$s_{K,m}(x,t)$};
\node[box, right=0.6cm of signal] (add)
{Training objective\\Equation~\eqref{eq:loss}};
\node[box, right=0.6cm of add] (model) {Watermarked\\model $v_{\theta,K,m}$};

\draw[arrow] (key) -- (signal);
\draw[arrow] (signal) -- (add);
\draw[arrow] (add) -- (model);

\node[font=\small\bfseries, anchor=west] (detlabel) at (-5.5, -2.5) {Detection (Black-Box)};

\node[box, below=0.7cm of detlabel] (query) {Query model\\$v(x_i, t_i)$};
\node[box, right=0.6cm of query] (proj) {Project\\$P^\top v$};
\node[box, right=0.6cm of proj] (demod) {Demodulate\\$\times\alpha(t)$};
\node[box, right=0.6cm of demod] (corr) {Correlate with\\codebook $C$};
\node[box, right=0.6cm of corr] (out) {Decoded\\message $\hat{m}$};

\draw[arrow] (query) -- (proj);
\draw[arrow] (proj) -- (demod);
\draw[arrow] (demod) -- (corr);
\draw[arrow] (corr) -- (out);

\draw[arrow, dashed, gray] (key.south) -- ++(0, -0.5) -| (proj.north);
\node[font=\footnotesize\color{gray}] at (0.4, -2.7) {Same $K$};

\end{tikzpicture}%
}

\caption{Overview of the proposed dynamics-level watermarking framework.
\textbf{Top (Training):} A message $m$ and secret key $K$ define a
watermark signal $s_{K,m}$. The embedding objective combines regression
to $\tilde u=u+a_{\mathrm{nom}}s_{K,m}$ with the correlation reward in
Equation~\eqref{eq:loss}. \textbf{Bottom (Detection):} Given black-box
access to a model, the detector spatially projects and synchronously demodulates velocity
queries, then applies maximum-correlation decoding to recover the embedded message. The same key $K$ is used for
embedding and detection.}
\label{fig:overview}
\end{figure}

\section{Method}
\label{sec:wmrc}
\subsection{Signal Model}

We model embedding as an additive velocity perturbation. Let $u(x,t)$ be the clean
velocity field for a $D$-dimensional state $x\in\mathbb R^D$ and time
$t\in[0,1]$, with watermarked counterpart
\begin{equation}\label{eq:vwm}
v_{\theta,K,m}(x,t) = u(x,t) + \eta(x,t) + \epsilon s_{K,m}(x,t),
\end{equation}
where $\eta(x,t)$ models training residuals and $s_{K,m}(x,t)$ is a key-dependent signal encoding
message $m$. Here $\epsilon>0$ is the effective carrier amplitude in the
velocity-field model; $a_{\mathrm{nom}}$ denotes the nominal amplitude prescribed
in the marked reference.

\subsection{Codebooks and Distortion Constraints}

To specify the signal $s_{K,m}$, let $m\in\mathcal M$ be a message from a set of
$|\mathcal{M}| := 2^L$ possible messages, where $L$ is the message length
in bits, and let $K$ be a secret key.
The key defines a codebook mapping $\phi_K:\mathcal M\to\mathcal H$,
where $\mathcal H$ is a function space over $(x,t)$. The codeword for
message $m$ is $s_{K,m}(x,t) := \phi_K(m)$. We equip $\mathcal H$ with the query-distribution $L^2$ norm,
$\|s\|_{\mathcal H}^2:=\mathbb E_{(x,t)\sim p_q}\|s(x,t)\|_2^2$.

To quantify distributional distortion, let $\rho(x,t)$ be the clean-flow
density and $\|f\|_{\TV}:=\int_{\mathbb R^D}|f(x)|\,dx$ the
total-variation norm of a signed density. Write $\rho^w_{K,m}$ for the
density induced by $u+\epsilon s_{K,m}$ from the same initial
distribution as the clean flow.

\begin{definition}[Admissible codebook]\label{def:admissible}
A codebook $\{s_{K,m}\}_{m\in\mathcal M}$ is \textbf{admissible} at strength
$\epsilon$ and terminal-distortion budget $\delta_{\mathrm{dist}}>0$,
relative to the clean flow, if it satisfies:
\textit{i)} bounded distributional drift,
$\|\rho^w_{K,m}(\cdot,1)-\rho(\cdot,1)\|_{\TV}\le\delta_{\mathrm{dist}}$
for every $m$;
\textit{ii)} message separation,
$\Delta_K:=\min_{m\ne m'}\|s_{K,m}-s_{K,m'}\|_{\mathcal H}>0$; and
\textit{iii)} energy normalization,
$\|s_{K,m}\|_{\mathcal H}^2=E_s$ for every $m$, with a common $E_s>0$.
\end{definition}

Theorem~\ref{thm:drift-main} gives conditions under which the output
distortion constraint holds.

\subsection{Spatial Codes and Temporal Carrier}\label{sec:concrete}

We construct hypercube and orthogonal spatial codes. The
secret key is $K=(P,C)$,
where $C$ maps each message to a unit codeword and
$P \in \mathbb{R}^{D \times k}$ is a matrix with orthonormal columns,
satisfying $P^\top P = I_k$, generated by QR decomposition of a random matrix
seeded by the key. The projection maps a codeword $c_m\in\mathbb R^k$ to the
$D$-dimensional velocity direction $Pc_m$, preserving its norm.

\textbf{Hypercube encoding.} For a message $m\in\{0,1\}^L$ with
$L\leq k$, let $e_1,\ldots,e_L$ be the first $L$ standard basis vectors
in $\mathbb R^k$ and set
\begin{equation}\label{eq:hypercube-codeword}
c_m=\frac{1}{\sqrt L}\sum_{j=1}^L(2m_j-1)e_j.
\end{equation}
The first $L$ coordinates take values $\pm1/\sqrt L$, and the remaining
$k-L$ coordinates are zero. The $2^L$ codewords are therefore the vertices
of an $L$-dimensional hypercube embedded in $\mathbb R^k$. Each bit
selects the sign on one axis, and every codeword has unit norm.
The keyed map $P$ embeds these axes as orthonormal velocity directions
$Pe_j$, giving the spatial signal $Pc_m$.

\textbf{Orthogonal encoding.} This encoding assigns mutually orthogonal
unit vectors to the $2^L$ complete messages. When $2^L\leq k$, QR decomposition of $2^L$ Gaussian-sampled vectors
in $\mathbb R^k$ produces unit codewords with
$\langle c_m,c_{m'}\rangle=0$ for $m\ne m'$.
This encoding supports at most $\lfloor\log_2 k\rfloor$ bits.

The temporal carrier $\alpha(t)$ modulates the spatial code.
Unless otherwise specified,
the proposed watermark uses $\alpha(t)=\sin(2\pi t)$; the time-independent
comparison method uses $\alpha(t)=1$. The watermark signal is
\begin{equation}\label{eq:watermark_signal}
s_{K,m}(x,t) = \alpha(t) \cdot P c_m.
\end{equation}
\begin{proposition}[Signal energy and separation]\label{prop:energy-sep}
For either encoding above, with the sinusoidal carrier and $t\sim\mathrm{Unif}[0,1]$, the signal in \eqref{eq:watermark_signal} has energy $E_s=1/2$ and minimum
separation $\Delta_K=1$ for orthogonal codes and $\Delta_K=\sqrt{2/L}$ for
hypercube codes.
\end{proposition}

\begin{proof}
Since $P^\top P = I_k$ and $\mathbb{E}[\alpha^2] = 1/2$,
$\|s_{K,m}\|^2 = \tfrac12\|Pc_m\|^2 = \tfrac12$, giving $E_s = 1/2$.
For separation,
$\|s_{K,m}-s_{K,m'}\|^2 = \tfrac12\|c_m-c_{m'}\|^2$.
Orthogonal codewords give $\|c_m-c_{m'}\|^2 = 2$, hence signal
distance $1$. For hypercube messages with Hamming distance $h$,
$\|c_m-c_{m'}\|^2 = 4h/L$, hence signal distance $\sqrt{2h/L}$, whose
minimum over $h\ge1$ is $\sqrt{2/L}$.
\end{proof}

\subsection{Watermarked Training Objective}

We embed either code by forming a marked velocity reference.
Let $\alpha:[0,1]\to\mathbb R$ be measurable with
$0<\mathbb E_t[\alpha(t)^2]<\infty$.
For a training pair $(x_0,x_1)$ and interpolated state
$x_t=(1-t)x_0+tx_1$, let $u(x_t,t)$ be the velocity of the clean model.
Here $U=x_1-x_0$ is the sample-pair flow-matching target, whereas
$u(x,t)=\mathbb E[U\mid x_t=x,t]$ is the population clean velocity.
Set $q=Pc_m$ and define the watermarked reference
$\tilde u=u+a_{\mathrm{nom}}\alpha(t)q$, where $a_{\mathrm{nom}}>0$.
Let $e_\theta=v_\theta-\tilde u$ and
$\Gamma(\theta)=\mathbb E[\alpha(t)
\langle P^\top(v_\theta-u),c_m\rangle]$. The training objective is
\begin{equation}\label{eq:loss}
\mathcal L(\theta)
=\mathbb E\!\left[
D^{-1}\|e_\theta\|_2^2+\frac{w_P}{k}\|P^\top e_\theta\|_2^2\right]
-\lambda_{\rm corr}\Gamma(\theta),
\qquad \lambda_{\rm corr}\geq0.
\end{equation}
The quadratic terms regress to $\tilde u$. The second quadratic measures
the error after projection by $P^\top$, so averaging over all $D$ outputs
does not underweight the $k$ keyed coordinates. The correlation term
increases the carrier-correlated coefficient along $c_m$ and may change
its amplitude. Proposition~\ref{prop:optimal-predictor}
shows that the population minimizer retains the prescribed carrier and
code direction. The primary CIFAR-10 experiments use
$\lambda_{\rm corr}=0$; Appendix~\ref{app:signal-strength} compares the
quadratic objectives, and Appendix~\ref{app:embedding-configurations}
specifies the cross-model settings with $\lambda_{\rm corr}>0$, including
the bounded implementation of the same correlation reward.

\begin{proposition}[Population solution of the embedding objective]
\label{prop:optimal-predictor}
Assume $\mathbb E\|\tilde u(x_t,t)\|_2^2<\infty$, $P^\top P=I_k$,
$\|c_m\|_2=1$, $w_P\geq0$, and
$0<\mathbb E[\alpha(t)^2]<\infty$. Equation~\eqref{eq:loss} has the
unique population minimizer
\begin{equation}
v^\star(x,t)=u(x,t)+\epsilon^\star\alpha(t)Pc_m,
\end{equation}
with
\begin{equation}\label{eq:optimal-amplitude}
\epsilon^\star
=a_{\mathrm{nom}}
+\frac{\lambda_{\rm corr}}
{2(D^{-1}+w_P/k)}.
\end{equation}
\end{proposition}

\begin{proof}
Let $q=Pc_m$ and $W=D^{-1}I+(w_P/k)PP^\top$. The first variation of
the reward is a scalar multiple of $\alpha(t)q$. Since
$Wq=(D^{-1}+w_P/k)q$, stationarity forces
$v-u=\epsilon\alpha(t)q$ for some scalar $\epsilon$. The objective then
reduces, up to a constant, to
\begin{equation}
\mathbb E[\alpha(t)^2]\left[
\left(\frac1D+\frac{w_P}{k}\right)
(\epsilon-a_{\mathrm{nom}})^2
-\lambda_{\rm corr}\epsilon\right].
\end{equation}
This strictly convex quadratic has the minimizer in
Equation~\eqref{eq:optimal-amplitude}.
\end{proof}

\begin{algorithm}[!htb]
\caption{Velocity-Field Watermark Embedding}
\label{alg:watermark}
\begin{algorithmic}[1]
\State \textbf{Input:} Clean reference $u$, query states $(x_t,t)$, key $K=(P,C)$, message $m$,
\State \hspace{\algorithmicindent} temporal carrier $\alpha(t)$, target coefficient $a_{\mathrm{nom}}$, weights $w_P$ and $\lambda_{\rm corr}$
\State Retrieve codeword $c_m$ from codebook $C$
\State Define watermarked reference: $\tilde u(x_t,t)\leftarrow u(x_t,t)+a_{\mathrm{nom}}\alpha(t)Pc_m$
\State Initialize $v_\theta$ from the clean checkpoint
\For{each training step}
    \State Sample a batch of query states and evaluate $\tilde u$
    \State Forward pass: $\hat{v}^{(i)} \leftarrow v_\theta(x_t^{(i)}, t^{(i)})$
    \State Evaluate Equation~\eqref{eq:loss} on the batch and update $\theta$
\EndFor
\State \textbf{Output:} Watermarked model $v_{\theta,K,m}$
\end{algorithmic}
\end{algorithm}

\paragraph{Balanced target.}
As an endpoint-preserving alternative, we replace the additive target by
$\tilde v_{\mathrm{bal}}(x,t)=u(x-\delta(t),t)+\dot\delta(t)$, where
$\delta(t)=a_{\mathrm{nom}}(1-\cos(2\pi t))Pc_m/(2\pi)$.
A standalone model is trained to regress this target using the frozen clean
velocity field $u$. The prescribed field translates each clean trajectory
by $\delta(t)$ and therefore has the same endpoint under exact integration
(Proposition~\ref{prop:endpoint-reference}).
Appendix~\ref{sec:matched-constructions} compares its learned implementation
with the additive target under matched fitting conditions.

\subsection{Synchronous Demodulation and Decoding}\label{sec:detection}

To recover the message, the owner uses the secret key $K=(P,C)$ and queries
the model $v$ at $N$ states and times $(x_i,t_i)$. The demodulated statistic is
\begin{equation}\label{eq:detection}
\hat{s} = \frac{1}{N}\sum_{i=1}^N \alpha(t_i) P^\top v(x_i, t_i).
\end{equation}
This matched-correlation statistic applies to any known scalar carrier
with $0<\mathbb E_{p_q}[\alpha(t)^2]<\infty$; zero temporal mean is not
required. Equation~\eqref{eq:finite-query-channel} gives its exact
finite-query decomposition under Assumption~\ref{ass:iid}.
The mean background score is
$b_K=\mathbb E_{p_q}[\alpha(t)P^\top v_{\mathrm{bg}}(x,t)]$, where
$v_{\mathrm{bg}}$ is the background component in the additive model.
The decoding statistic is $s_{\mathrm{dec}}=\hat s-\widehat b_K$ when a
fixed clean-reference estimate is used, and $s_{\mathrm{dec}}=\hat s$ for
raw decoding. The SD3.5 and audio protocols estimate clean projections during embedding;
their training targets and raw-decoding procedures are specified in
Appendix~\ref{app:embedding-configurations}.

The embedded message is recovered by maximum-correlation decoding:
\begin{equation}
\hat{m} = \arg\max_{m \in \mathcal{M}} \langle s_{\mathrm{dec}}, c_m \rangle.
\end{equation}
An explicit codebook requires maximum-correlation search; hypercube
decoding reduces to coordinatewise sign decisions.
Proposition~\ref{prop:decoder-consistency} proves consistency for any
finite set of distinct unit-norm codewords.

\begin{algorithm}[!htb]
\caption{Velocity-Query Watermark Decoding}
\label{alg:detection}
\begin{algorithmic}[1]
\State \textbf{Input:} Model $v$, key $K=(P,C)$, carrier $\alpha$, query count $N$, distribution $p_q(x,t)$, optional fixed background estimate $\widehat b_K$
\State Sample $(x_i,t_i)\sim p_q$ independently, with $t_i\sim\mathrm{Unif}[0,1]$
\State Query model: $y_i \leftarrow v(x_i, t_i)$ for $i = 1, \dots, N$
\State Project to code space: $z_i \leftarrow P^\top y_i \in \mathbb{R}^k$
\State Demodulate: $d_i \leftarrow \alpha(t_i)z_i$
\State Average demodulated projections: $\hat{s} \leftarrow \frac{1}{N}\sum_{i=1}^N d_i \in \mathbb{R}^k$
\State Set $s_{\mathrm{dec}}\leftarrow\hat s-\widehat b_K$ (use $\widehat b_K=0$ for raw decoding)
\State Decode message: $\hat{m} \leftarrow \arg\max_{m \in \mathcal{M}} \langle s_{\mathrm{dec}}, c_m \rangle$
\State \textbf{Output:} Decoded message $\hat{m}$ and correlation score $\max_m \langle s_{\mathrm{dec}}, c_m \rangle$
\end{algorithmic}
\end{algorithm}

Under the additive model,
$\mathbb E[\hat s]=\epsilon\kappa_\alpha c_m+b_K$, where
$\kappa_\alpha=\mathbb E_{p_q}[\alpha(t)^2]$. The centered background contribution has root-mean-square magnitude
$O(N^{-1/2})$ under independent finite-variance queries.
The background estimate contributes a separate calibration error.

\subsection{Message Length and Query Count}

To relate query error to message length, we model the averaged code scores
using an additive Gaussian approximation~\citep{cover2006elements}.

For a general temporal carrier $\alpha(t)$, recall
$\kappa_\alpha=\mathbb E_{p_q}[\alpha(t)^2]\in(0,\infty)$.
The demodulated signal gain is $\epsilon\kappa_\alpha$. Let
$\Sigma_\alpha$ denote the covariance of the single-query demodulated
background projection.

\begin{theorem}[Information bound for message recovery]
\label{thm:capacity}
Fix a key $K$ and let $M$ be drawn from a finite message set with unit-norm
codewords $c_M\in\mathbb R^k$. Consider the Gaussian approximation
\begin{equation}\label{eq:general-demodulated-channel}
\hat s=\epsilon\kappa_\alpha c_M+b_K+\xi,
\qquad \xi\sim\mathcal N(0,\Sigma_\alpha/N),
\end{equation}
where $b_K$ is fixed and message-independent, $\Sigma_\alpha\succ0$, and
$\xi$ is independent of $M$. For $Q=\operatorname{Cov}(c_M)$,
\[
I(M;\hat s)\le\min\left\{H(M),\frac12\log\det\!\left(
I_k+N\epsilon^2\kappa_\alpha^2
\Sigma_\alpha^{-1/2}Q\Sigma_\alpha^{-1/2}\right)\right\}.
\]
Allowing Gaussian input vectors with covariance $Q$ and
$\operatorname{tr}Q\le1$ gives
\begin{equation}
\mathcal C_\alpha=
\max_{\substack{Q\succeq0\\\operatorname{tr}Q\le1}}
\frac12\log\det\!\left(I_k+N\epsilon^2\kappa_\alpha^2
\Sigma_\alpha^{-1/2}Q\Sigma_\alpha^{-1/2}\right).
\end{equation}
\end{theorem}

Appendix~\ref{app:capacity_proof} derives the Gaussian approximation and
proves the bound.

For uniform query time, the sinusoidal carrier has $\kappa_\alpha=1/2$,
recovering the factor $N\epsilon^2/4$ in the log-determinant expression.
The time-independent carrier has $\kappa_\alpha=1$ and its own background
covariance. The maximization over $Q$ gives the corresponding Gaussian-channel
capacity under an average-power constraint~\citep{cover2006elements}.

For uniform hypercube messages, apply sign decoding to the centered
statistic $Y=\hat s-b_K$. Under isotropic Gaussian noise
$\Sigma_\alpha=\sigma^2I_k$, meaning independent equal-variance noise in
the code coordinates, message recovery has probability
\begin{equation}\label{eq:hypercube-exact-recovery}
P_{\mathrm{exact}}(L)=
\left[\Phi\!\left(
\frac{\epsilon\kappa_\alpha\sqrt N}{\sigma\sqrt L}\right)\right]^L,
\end{equation}
where $\Phi$ is the standard-normal CDF. At fixed carrier $\alpha(t)$, $N$, $\epsilon$,
and $\sigma$, increasing $L$ reduces per-bit signal energy and compounds bit errors
in complete-message decoding.

For uniform orthogonal messages with $2^L=k$, the codeword covariance is
$Q=I_k/k-\bar c\bar c^T$, with
$\bar c=k^{-1}\sum_m c_m$. Uniform hypercube messages have
$Q=L^{-1}\sum_{j=1}^L e_je_j^T$. The finite-codebook bound uses each
codebook's covariance, while decoding reliability also depends on its
geometry. Section~\ref{sec:experiments} evaluates the message length effect at
fixed carrier and nominal signal strength.

\subsection{Terminal Distortion}\label{sec:preservation}

The carrier amplitude $\epsilon$ affects message recovery and output
distortion. A transport-stability argument gives the
following bound in terms of a density-weighted forcing coefficient.
Appendix~\ref{app:preservation} gives the proof and includes the learned
residual in the same analysis.

\begin{theorem}[Distributional stability]\label{thm:drift-main}
Let $\rho$ and $\rho^w_{K,m}$ be generated by $u$ and
$u+\epsilon s_{K,m}$ from the same initial density, with $\epsilon>0$.
Assume that the continuity equations are represented by invertible flow
maps and admit an integral representation for equations with source terms on
$[0,1]$, and that
\begin{equation}
C_{K,m}:=\int_0^1
\bigl\|\nabla\!\cdot(\rho(\cdot,t)s_{K,m}(\cdot,t))\bigr\|_{\TV}\,dt
<\infty.
\end{equation}
Then
\begin{equation}
\bigl\|\rho^w_{K,m}(\cdot,1)-\rho(\cdot,1)\bigr\|_{\TV}
\le\epsilon C_{K,m}.
\end{equation}
For either encoding in Section~\ref{sec:concrete}, with
$0<\kappa_\alpha=\mathbb E_{p_q}[\alpha(t)^2]<\infty$, the codebook is
admissible whenever
$\epsilon\max_m C_{K,m}\le\delta_{\mathrm{dist}}$ and the flow assumptions
hold for every message.
\end{theorem}

\begin{proof}
Theorem~\ref{thm:drift} in Appendix~\ref{app:preservation} gives the
finite-amplitude TV bound. The stated condition therefore establishes
the distortion requirement. Both encodings have distinct unit-norm
codewords, common signal energy $E_s=\kappa_\alpha$, and pairwise signal
distances $\sqrt{\kappa_\alpha}\|c_m-c_{m'}\|_2$, establishing the remaining
conditions in Definition~\ref{def:admissible}.
\end{proof}

The bound identifies how the carrier amplitude $\epsilon$ and the
density-weighted divergence $C_{K,m}$ control distributional drift.
Endpoint NMSE quantifies distortion for the trained models.

For the sinusoidal carrier, Corollary~\ref{cor:zero-mean-gap} characterizes
first-order cancellation through the transported density perturbation.

Appendix~\ref{app:endpoint-controls} treats paired endpoints separately:
it gives an exactly preserving reference field, bounds fitting and solver
errors, and compares learned additive and balanced constructions.

\section{Model Access, Removal, and Transfer}\label{sec:security_analysis}

We analyze key identifiability, carrier removal, and watermark transfer
under recipient access to the released model.

\subsection{Threat Model}

Verification assumes that an auditor holding the key can submit states
and times to a suspect model and receive $v(x,t)$. An illustrative deployment
is a hosted flow service that exposes velocity evaluations for
client-selected ODE solvers or model distillation while keeping its
parameters private. The recipient may inspect or modify a released model
or train a replacement from its velocities or generated samples. The
embedding and decoding algorithms are public; only the key is secret.
We evaluate removal and transfer under these access levels, following
model-ownership watermarking formulations~\citep{uchida2017embedding,adi2018turning}.

The sinusoidal carrier is public and secrecy resides in the spatial key
$(P,C)$; carrier subtraction evaluates removal under this setting.

\subsection{Key Ambiguity and Clean-Model FPR}\label{sec:key-identifiability}

We first consider key recovery from an observed watermark direction.
For a fixed embedded message, the carrier modulates a single $D$-dimensional
direction $q=Pc_m$. The factorization of $q$ into projection and codeword is nonunique. The following propositions characterize subspace ambiguity and orthogonal
equivalence of the key coordinates.

\begin{proposition}[Single-message subspace ambiguity]
\label{prop:subspace_search}
For $q\ne0$ and $1<k<D$, the family of $k$-dimensional subspaces
compatible with the direction $q$ is parameterized by the Grassmannian
$\operatorname{Gr}(k-1,D-1)$, the set of $(k-1)$-dimensional subspaces of
a $(D-1)$-dimensional space, and has dimension
$(k-1)(D-k)$.
\end{proposition}

\begin{proof}
Let $\mathcal U_q=\operatorname{span}(q)$. The quotient space
$\mathbb{R}^D/\mathcal U_q$ has dimension $D-1$. For every
$k$-dimensional subspace $S$ containing $\mathcal U_q$, the quotient
$S/\mathcal U_q$ is a $(k-1)$-dimensional subspace of
$\mathbb{R}^D/\mathcal U_q$, and conversely the inverse image of every
such subspace is a $k$-dimensional subspace containing $\mathcal U_q$.
The correspondence is bijective. The dimension of
$\operatorname{Gr}(k-1,D-1)$~\citep{zheng2002communication} is
$(k-1)((D-1)-(k-1))=(k-1)(D-k)$.
\end{proof}

\begin{proposition}[Orthogonal key ambiguity]\label{prop:P_hidden}
For every orthogonal $R_o\in\mathbb{R}^{k\times k}$, the key pairs
$(P,C)$ and $(P R_o,\{R_o^T c_j\}_{j\in\mathcal M})$ produce the same
watermark signal for every message $j$.
\end{proposition}

\begin{proof}
For each $j$, $(P R_o)(R_o^T c_j)=P(R_oR_o^T)c_j=P c_j$.
Temporal modulation preserves this equivalence.
\end{proof}

\begin{corollary}[Clean-model false-positive rate]
\label{cor:clean-fpr}
Suppose the random-key distribution is independent of a clean model's
responses and has the following symmetry: for any two code labels, a key
transformation maps one label to the other without changing the key
distribution. Then the decoded clean-model label is uniform, and the false-positive
probability for any fixed $L$-bit ownership claim is $2^{-L}$.
\end{corollary}

\begin{proof}
The stated symmetry assigns equal probability to every decoded label, and
the probabilities sum to one. Coordinate sign flips give this symmetry for
the hypercube, while codeword-label permutations give it for the orthogonal
code.
\end{proof}

\subsection{Removal and Transfer}\label{sec:transfer}

A recipient can suppress recovery by targeting the observable
carrier-correlated component directly. Least-squares regression estimates
the component of the watermarked velocity correlated with the public carrier
$\alpha(t)=\sin(2\pi t)$:
\begin{equation}\label{eq:carrier-removal}
\widehat r=\frac{\sum_i \alpha(t_i)v_w(x_i,t_i)}{\sum_i \alpha(t_i)^2},
\qquad v_\beta(x,t)=v_w(x,t)-\beta\alpha(t)\widehat r.
\end{equation}
The regression estimates the carrier-correlated component in the full
velocity space; Equation~\eqref{eq:detection} measures its keyed projection.
The estimate includes both the watermark and carrier-correlated background
and requires only the public carrier and suspect-model queries.
A recipient may also continue ordinary flow-matching training on clean data
while updating all model parameters. Section~\ref{sec:modification} evaluates parameter
modification and signal subtraction. Distillation fits a student to source
responses~\citep{hinton2015distilling}: velocity-supervised students fit
source velocities, whereas sample-only students learn from generated terminal
samples. Section~\ref{sec:comparative-transfer} and
Appendix~\ref{app:sample-transfer} evaluate message recovery under these two
forms of supervision.

\section{Experiments}\label{sec:experiments}
\subsection{Setup and Metrics}

\paragraph{Models and datasets.}
The cross-model study evaluates image MLP, RF-UNet, and SD3.5 DiT generators
on MNIST, CIFAR-10, CelebA-64, and image latents. Audio experiments evaluate
an MLP and a Transformer with fixed Music2Latent
representations~\citep{music2latent}, together with Stable Audio 3 Small
Music Base and its codec~\citep{stableaudio3base}. Audio data comprise
six-second FMA-small recordings with the official recording-level
split~\citep{fmadataset}.

The similar recovery and quality-preservation results across image and
audio architectures motivate using the pretrained CIFAR-10 RF-UNet as a
representative testbed for detailed payload, modification, removal, and
transfer experiments.

\paragraph{Encoding and optimization.}
The cross-model study in Tables~\ref{tab:main} and~\ref{tab:all-generators}
uses projection dimension $k=32$ and five messages per setting, with
five-bit orthogonal (O5), five-bit hypercube (H5), and 32-bit hypercube
(H32) encodings. The tables report mean and population standard deviation
across messages. These deviations describe the fixed evaluated message
sets rather than confidence intervals.

The settings in these tables use the same spatial-code construction, sinusoidal carrier,
and the same synchronous decoder. Optimization terms used for the larger
pretrained models are stated in Appendix~\ref{app:embedding-configurations};
the primary CIFAR-10 studies use Equation~\eqref{eq:loss}.

For cross-model feasibility and payload scaling, pretrained generators use
low-rank adaptation (LoRA)~\citep{hu2022lora} and selected additional layers
(Appendix~\ref{app:embedding-configurations}); scratch-trained models
optimize all weights. Parameter-efficient embedding limits adaptation
cost. CIFAR-10 robustness, the four-method comparison, and their students
use full-weight optimization, allowing modification of the entire released
model. The CIFAR-10 carrier study also updates all weights, with the same
trainable parameters across designs. Table~\ref{tab:embedding-scope-comparison}
compares LoRA and full-weight embedding under matched whole-model
modifications; Appendix~\ref{app:restricted-modifications} separately
reports modifications restricted to embedding parameters.

\paragraph{Recovery and quality metrics.}
Message recovery is the proportion of verification trials that reconstruct
the complete message; bit accuracy is the fraction of correctly decoded
bits. Clean-model FPR is the probability that a clean model is accepted for
a fixed ownership claim under the random-code construction. Its estimation
protocol is specified in Appendix~\ref{app:fpr}.
The quality ratio is the watermarked-to-clean ratio of Fr\'echet
feature distances to a common reference. Image
experiments use Inception FID; audio experiments use Fr\'echet distance in
CLAP feature space~\citep{wu2023clap}. References are real images for
MNIST and CIFAR-10, base-model samples for CelebA-64 and SD3.5, and held-out
FMA recordings for audio. Ratios near one indicate similar feature distances.
CIFAR-10 quality studies use 500 images and 100 sampling steps unless the
study specifies a different sample count. The four-method comparison uses
5,000 images. The paired random-initialization and corrected-sample studies
use 1,000 images per model and their specified matched references.
Paired image distortion compares clean and watermarked outputs generated
from the same initial noise and condition. It is measured by normalized
mean-squared error (NMSE), the error energy relative to clean output energy in RGB $[0,1]$
(Equation~\eqref{eq:relative-energy}). All conditions within a study
share the same clean-reference denominator.

Paired final-sample errors, their numerical-solver reference, and fixed mean correction are
reported together in Appendix~\ref{app:endpoint-controls}; the correction
subtracts a fixed RGB mean offset estimated from paired outputs.

\subsection{Recovery, Generation Quality, and Cost}

Table~\ref{tab:main} reports cross-modal 32-bit recovery and generation quality.

\begin{table}[!htb]\centering\small
\caption{32-bit hypercube recovery and generation quality at $k=32$.
Mean and population standard deviation across five messages; evaluation
sample counts are specified in Appendix~\ref{app:recovery-quality}.}
\label{tab:main}
\begin{tabular}{llrr}\toprule
Generator & Data & Recovery (\%) & Quality ratio \\
\midrule
Image MLP & MNIST & $100.0{\pm}0.0$ & $1.041{\pm}0.022$ \\
RF-UNet & MNIST & $100.0{\pm}0.0$ & $1.014{\pm}0.012$ \\
RF-UNet & CIFAR-10 & $100.0{\pm}0.0$ & $1.032{\pm}0.021$ \\
RF-UNet & CelebA-64 & $100.0{\pm}0.0$ & $1.036{\pm}0.025$ \\
SD3.5 DiT & Image latents & $100.0{\pm}0.0$ & $0.992{\pm}0.019$ \\
Audio MLP & FMA & $100.0{\pm}0.0$ & $1.003{\pm}0.002$ \\
Audio Transformer & FMA & $100.0{\pm}0.0$ & $0.999{\pm}0.004$ \\
Stable Audio & FMA & $100.0{\pm}0.0$ & $0.933{\pm}0.005$ \\
\bottomrule\end{tabular}
\end{table}

Table~\ref{tab:main} shows complete recovery for all evaluated 32-bit
messages. The quality ratios are descriptive comparisons with matched
clean controls; their evaluation sample counts are given in the appendix.
Table~\ref{tab:all-generators} includes both code families and clean-model
label matches. In the fresh-key MLP experiment, the observed five-bit FPR is
approximately 3.1\%, and no H32 match is observed
(Appendix~\ref{app:fpr}).

For the primary full-weight H32 models, image NMSE is 0.162\% before
mean correction, which subtracts a fixed RGB offset estimated on 256 paired
outputs, and 0.111\% afterward, against 0.0132\% for the clean
Euler-100/RK4-100 difference (Appendix~\ref{sec:endpoint-study}).
In the matched parameterization comparison, LoRA-embedded O5 and H32 models
have paired endpoint NMSE of 1.446\% and 1.438\%, respectively, while their
FID/clean ratios remain 1.025 and 1.030
(Appendix~\ref{app:matched-embedding}); paired error and distributional
feature quality therefore measure distinct effects of embedding.
A separate comparison fitted under matched conditions, using a lower learning
rate, gives NMSE of 0.0441\%, 0.0511\%, and
0.0572\% for additive sine, balanced sine, and the time-independent
signal, respectively (Appendix~\ref{sec:matched-constructions}).
Both sinusoidal variants achieve complete velocity-message recovery.
The balanced field's exact endpoint identity leaves fitting and solver
errors after implementation; it provides no observed fidelity advantage
under these conditions.

\paragraph{Embedding and verification budgets.}\label{sec:embedding-cost}
To quantify embedding and verification cost, we evaluate embedding-step
and query-count sweeps using LoRA-embedded CIFAR-10 UNet O5 models.
The earliest evaluated checkpoint with complete recovery for all messages
is 30 optimization steps, with FID ratio $1.077\pm0.044$ at that checkpoint.
A query-count sweep on the evaluated five-bit model gives complete recovery
at 32 queries.
The main CIFAR-10 recovery experiments use 4096 queries per trial.
\subsection{Carrier Design}\label{sec:carrier-design}

We compare the single sinusoid with two four-mode designs on the CIFAR-10
RF-UNet: frequency-division multiplexing (FDM-4), using four sinusoidal
frequencies, and time-division multiplexing (TDM-4), using four disjoint
temporal windows.
All conditions use the same four spatial directions,
payload, time-averaged perturbation energy, initialization, training queries,
and objective; only the temporal waveforms differ. We use three independent
projection/message keys.
Table~\ref{tab:temporal-multiplexing} gives the complete comparison and
Appendix~\ref{app:multiplexing} specifies the normalized carriers.

The primary outcome is held-out error in the prescribed keyed coefficients.
After 1000 updates, the keyed-coefficient NMSE is $1.096\pm0.030\%$ for the single
sinusoid, $4.690\pm0.059\%$ for FDM-4, and $8.153\pm0.152\%$ for TDM-4.
All messages are recovered; keyed-coefficient NMSE additionally resolves
differences in carrier fitting accuracy. The single sinusoid is
both fitted more accurately and produces lower endpoint error than the tested
four-mode carriers. Its endpoint NMSE is not directly comparable with the
primary H32 value because this four-direction experiment uses prescribed
signal energy $A_0^2=2.25$, twice the primary value $a_{\mathrm{nom}}^2/2=1.125$.
Appendix~\ref{app:signal-strength} further separates the recovered carrier
component from the remaining velocity change introduced by training.

\subsection{Payload Scaling}

To examine scaling beyond 32 bits, Table~\ref{tab:payload-main} varies the
hypercube message length using LoRA embedding in the pretrained CIFAR-10
RF-UNet at fixed projection dimension
$k=128$, nominal signal strength
$a_{\mathrm{nom}}=1.5$, and query count $N=4096$.

\begin{table}[!htb]
\centering
\caption{Hypercube message length scaling at $k=128$ and $N=4096$.
Recovery uses 40 trials per message.}
\label{tab:payload-main}
\small
\begin{tabular}{cccc}
\toprule
Bits & Message recovery ($\uparrow$) & Bit accuracy ($\uparrow$) & FID ratio ($\downarrow$) \\
\midrule
64 & $93.5{\pm}8.3$\% & $99.90{\pm}0.13$\% & $1.052{\pm}0.014$ \\
96 & $16.0{\pm}22.9$\% & $98.18{\pm}0.81$\% & $1.045{\pm}0.014$ \\
128 & $1.0{\pm}2.0$\% & $97.07{\pm}1.27$\% & $1.044{\pm}0.009$ \\
\bottomrule
\end{tabular}
\end{table}

In Table~\ref{tab:payload-main}, bit accuracy remains high while
message recovery decreases from 64 to 128 bits. At fixed total signal
energy, per-bit amplitude scales as $L^{-1/2}$;
Appendix~\ref{app:query-scaling} shows that finite-query variation and
the fitted model's mean bit scores jointly determine recovery.
Theorem~\ref{thm:capacity} bounds the information in the demodulated
statistic; Equation~\eqref{eq:hypercube-exact-recovery} characterizes the
payload dependence under isotropic Gaussian noise.

\subsection{Modification and Signal Removal}
\label{sec:modification}\label{sec:removal-transfer}

We test persistence under model modification using five O5 and five H32
messages embedded by updating all CIFAR-10 UNet weights. Each modification
starts independently from its watermarked checkpoint. Clean-data fine-tuning
continues flow-matching training on clean data while updating all parameters
for 1,000 steps. Carrier
subtraction estimates the sine-correlated velocity component from 16,384
queries; regression against cosine, whose inner product with the sine carrier
is zero on $[0,1]$, serves as a control. Table~\ref{tab:modification-summary}
summarizes recovery and FID. Appendix~\ref{app:full-weight-studies} gives the
training settings and full results, including standard deviations,
stronger compression, and the lower fine-tuning learning rate.

\begin{table}[!htb]\centering\small
\caption{Whole-model modification and signal subtraction after full-weight embedding. Recovery is
message recovery / bit accuracy (\%), averaged across five messages per code.
FID is normalized by the original clean model.}
\label{tab:modification-summary}
\begin{tabular}{lrrrr}\toprule
Condition & O5 recovery & H32 recovery & O5 FID ratio & H32 FID ratio \\
\midrule
Original & 100.0/100.0 & 100.0/100.0 & 1.008 & 1.006 \\
25\% pruning & 100.0/100.0 & 100.0/100.0 & 1.012 & 1.006 \\
4-bit rounding & 100.0/100.0 & 100.0/100.0 & 1.159 & 1.157 \\
Fine-tuning, $10^{-4}$ & 77.0/88.4 & 0.0/71.3 & 0.968 & 0.967 \\
Sine subtraction & 2.0/49.4 & 0.0/48.0 & 1.007 & 1.007 \\
Cosine subtraction & 100.0/100.0 & 100.0/100.0 & 1.008 & 1.006 \\
\bottomrule\end{tabular}
\end{table}

Table~\ref{tab:modification-summary} shows complete recovery after 25\%
pruning and 4-bit quantization; quantization increases the FID ratio. Clean-data
fine-tuning reduces recovery while the FID ratio remains below one. Sine subtraction
reduces bit accuracy to approximately 50\% while retaining near-unit FID ratios. Cosine
subtraction retains complete recovery, consistent with sine--cosine
orthogonality under uniform query-time sampling.

A matched comparison isolates the embedding parameterization while
retaining whole-model modification budgets
(Table~\ref{tab:embedding-scope-comparison}). Both LoRA and full-weight
embedding retain complete recovery after 25\% pruning and 4-bit rounding,
and both are suppressed by sine-carrier subtraction. Their fine-tuning
sensitivity differs: after 1000 whole-model fine-tuning updates at $10^{-4}$, O5
message recovery is 77\% for full-weight embedding and 6\% for LoRA
embedding; H32 bit accuracy is 71.3\% and 51.9\%, respectively. Thus the
two embedding strategies share the tested compression and removal behavior,
while residual recovery after clean-data fine-tuning depends on the
embedding parameterization.

\subsection{Comparing Verification and Transfer across Four Methods}
\label{sec:comparative-transfer}

The methods differ principally in their verification interfaces. The
velocity watermark is read without parameter access and is unchanged by
function-preserving parameter reparameterizations; the comparison below
separately measures recovery under modification and transfer.

We compare verification interfaces using fixed
32-bit source models with complete message recovery. The proposed sinusoidal velocity signal
is compared with three ownership methods: a weight projection adapted from \citet{uchida2017embedding}, a trigger response adapted from \citet{adi2018turning}, and a time-independent velocity signal inspired
by spread-spectrum signaling~\citep{cox1997spread}. All four embed 32-bit
messages by updating the complete CIFAR-10 UNet. Verification uses parameter projection, trigger-conditioned outputs, or
demodulated velocity correlations, according to the embedding method. Appendix~\ref{app:comparative-study} specifies the objectives
and evaluation protocol.

Table~\ref{tab:comparative-transfer} evaluates the four methods under the
same modification types and two forms of student supervision. Students
share the source architecture, start from the public clean checkpoint,
and update all parameters for 1,000 steps. Velocity-supervised students fit
source velocities; sample-only students train on 2,000 generated images,
with images from the clean source used for the corresponding controls.
The original models have FID ratios of 1.000, 1.003, 1.006, and 1.052,
respectively, in table-column order.

\begin{table}[!htb]\centering\small
\caption{Four-method comparison on CIFAR-10 with 32-bit messages. Cells give
message recovery / bit accuracy (\%). Stochastic entries average
three training seeds; sample-only students use uncorrected source outputs.}
\label{tab:comparative-transfer}
\setlength{\tabcolsep}{4pt}
\begin{tabular}{@{}lrrrr@{}}\toprule
Condition & Velocity field & Weight & Trigger & \shortstack{Time-independent\\signal} \\
\midrule
Original & 100/100.0 & 100/100.0 & 100/100.0 & 100/100.0 \\
\midrule
25\% pruning & 100/100.0 & 100/100.0 & 100/100.0 & 100/100.0 \\
4-bit rounding & 100/100.0 & 100/100.0 & 100/100.0 & 100/100.0 \\
Fine-tuning, $10^{-4}$ & 0/73.3 & 100/100.0 & 0/56.3 & 0/46.1 \\
\midrule
Velocity-supervised student & 100/100.0 & 100/100.0 & 0/59.0 & 100/100.0 \\
Clean-data sample student & 0/52.1 & 0/53.1 & 0/56.3 & 0/41.9 \\
Sample-only student & 0/75.8 & 0/61.5 & 0/56.1 & 98.3/99.9 \\
\bottomrule\end{tabular}
\end{table}

The velocity-field fine-tuning entry is the fixed source used in this
comparison; Table~\ref{tab:modification-summary} reports the corresponding
mean across five messages.

All four methods retain complete recovery after 25\% pruning and 4-bit
rounding. Clean-data fine-tuning reduces recovery for the sinusoidal,
trigger, and time-independent signals. The weight-based baseline retains
complete message recovery. Under velocity supervision, both velocity signals transfer
completely, as does the weight mark. Under
sample-only supervision, the time-independent signal retains near-complete
recovery. This spread-spectrum-inspired signal is present in generated
samples and can therefore be retained by sample-only training. In matched
fitted models, removing only the prescribed carrier changes endpoint energy
by 14.081\% of its signal energy for the time-independent carrier, compared
with 1.414\% for the sinusoid
(Appendix~\ref{sec:matched-constructions}). This supports stronger coupling
to terminal samples as a transfer mechanism. The other
methods have zero message recovery; the sinusoidal method retains 75.8\%
bit accuracy, compared with 52.1\% for its clean-data control. Mean output
correction removes most of this excess agreement
(Appendix~\ref{app:uncorrected-transfer}).
Appendix~\ref{app:uncorrected-transfer} reports sample-only quality and the
additional multi-message study.

\subsection{Velocity Transfer and Functional Invariance}
\label{sec:functional-invariance}

Across three random initializations, the same-architecture velocity students
recover all 32 bits, while their clean-source controls obtain
$47.86\pm0.97\%$ bit accuracy. Students trained to match the velocities
of the weight-watermarked source obtain
$44.79\pm3.90\%$ bit accuracy and no message recovery. For the reference
initialization, the velocity student's FID is 1.408 times the source FID;
relative velocity MSE is approximately 2\% across initializations
(Appendix~\ref{app:random-initialization}).
Under 20 hidden-channel permutations that leave the represented function
unchanged, source and student retain
100\% velocity-message recovery, whereas weight-bit accuracy falls to
55.000\% and 53.594\%, respectively
(Table~\ref{tab:interface-control}).

\section{Conclusion}\label{sec:discussion}\label{sec:scientific-scope}

We establish the feasibility of keyed multi-bit watermarking in learned
flow-matching velocity fields, with recovery across image and audio models
and generation quality comparable to matched clean controls. The analysis connects code geometry
and query count to recovery and bounds distributional drift. The balanced
reference preserves endpoints under exact integration, and its learned
CIFAR-10 implementation has distortion comparable to additive sine. The
single sinusoid is also learned more accurately than the tested four-mode
carriers. Recovery persists under velocity supervision and
function-preserving channel permutations, while carrier subtraction and
clean-data fine-tuning suppress it. The four-method comparison distinguishes
verification of the represented velocity field from weight-mark retention
under the tested fine-tuning. The method therefore verifies the represented
velocity field without parameter access for non-adversarial
ownership auditing.

\noindent Code is available at
\url{https://github.com/ShuchanWang/flow-matching-dynamics-watermarking}.

\bibliographystyle{plainnat}
\bibliography{ref}

\appendix

\section{Endpoint Distortion and Output Distribution Stability}
\label{app:endpoint-controls}

We distinguish changes in the output distribution from paired endpoint
displacement. We first bound distributional drift, then present an
exactly endpoint-preserving reference field and trajectory stability
bounds. Experiments measure the sinusoidal carrier response separately
from the total output change and its reduction by fixed mean correction.

\subsection{Changes to the Output Distribution}
\label{app:preservation}

We derive the bound on output distribution change and then evaluate the
cancellation produced by the sinusoidal carrier.


Let $u : \R^D \times [0,1] \to \R^D$ be the clean velocity field and
$\varphi_t$ its flow map. Under the noise distribution $p_0$, the flow
induces a density $\rho(\cdot, t) = (\varphi_t)_\# p_0$ satisfying
\begin{equation}
\partial_t \rho + \nabla\!\cdot\!(\rho\,u) = 0,
\label{eq:ce-clean}
\end{equation}
with generative distribution $\rho(\cdot, 1) = (\varphi_1)_\# p_0$.
Let $\delta u(x,t)=v(x,t)-u(x,t)\in\mathbb R^D$ be the velocity
perturbation, with $\varphi_t^w$ and $\rho^w$ denoting the watermarked
flow map and density.
Write
$\mathcal T_{t\leftarrow r}f=(\varphi_{t\leftarrow r})_\#f$ and
$\mathcal T^w_{t\leftarrow r}f=(\varphi^w_{t\leftarrow r})_\#f$ for
signed-density transport by the clean and perturbed flows. We use the
flow representation of continuity equations~\citep{diperna1989ordinary,ambrosio2004transport}.
The fields and densities are assumed regular enough for the continuity
equations and the integral solution formulas used below; smooth,
globally defined invertible flows and integrable source terms suffice.

\paragraph{Distribution drift at finite strength.}

\begin{theorem}[TV stability under a velocity perturbation]
\label{thm:drift}
Let $v=u+\delta u$ and assume that the clean and perturbed continuity
equations have the flow representations above, with the same initial
density. If $\nabla\!\cdot(\rho\,\delta u)$ is integrable in time with
values in $L^1$, then
\begin{equation}\label{eq:full-drift-bound}
\|\rho^w(\cdot,1)-\rho(\cdot,1)\|_{\TV}
\le\int_0^1\|\nabla\!\cdot(\rho(\cdot,t)\delta u(\cdot,t))\|_{\TV}\,dt.
\end{equation}
For $\delta u=\epsilon s_{K,m}$, this gives
\begin{equation}\label{eq:general-drift-bound}
\|\rho^w(\cdot,1)-\rho(\cdot,1)\|_{\TV}
\le\epsilon\int_0^1
\|\nabla\!\cdot(\rho(\cdot,t)s_{K,m}(\cdot,t))\|_{\TV}\,dt.
\end{equation}
For the proposed carrier $s_{K,m}(x,t)=\alpha(t)q$, $q=Pc_m$, the bound is
\begin{equation}\label{eq:drift-bound}
\|\rho^w(\cdot,1)-\rho(\cdot,1)\|_{\TV}
\le\epsilon\int_0^1|\alpha(t)|\,
\|q^\top\nabla\rho(\cdot,t)\|_{\TV}\,dt.
\end{equation}
\end{theorem}

\begin{proof}
Set $f=\rho^w-\rho$. Subtracting the continuity equations and grouping
transport under $v=u+\delta u$ gives the exact equation
\[
\partial_t f+\nabla\!\cdot(vf)
=-\nabla\!\cdot(\rho\,\delta u),\qquad f(\cdot,0)=0.
\]
Its flow representation is
\[
f(\cdot,1)=-\int_0^1\mathcal T^w_{1\leftarrow t}
[\nabla\!\cdot(\rho\,\delta u)(\cdot,t)]\,dt.
\]
An invertible flow transports the positive and negative parts of a signed
density without changing their total masses and therefore preserves its TV
norm.
Taking norms and applying the triangle inequality gives
\eqref{eq:full-drift-bound}. Substituting $\delta u=\epsilon s_{K,m}$
gives \eqref{eq:general-drift-bound}; the identity
$\nabla\!\cdot(\rho\alpha(t)q)=\alpha(t)q^\top\nabla\rho$
gives \eqref{eq:drift-bound}.
\end{proof}

For the fitted field $v_w=u+\epsilon s_{K,m}+\eta$, the same bound gives
\begin{equation}\label{eq:residual-tv-bound}
\|\rho^w(\cdot,1)-\rho(\cdot,1)\|_{\TV}
\le\epsilon C_{K,m}
+\int_0^1\|\nabla\!\cdot(\rho(\cdot,t)\eta(\cdot,t))\|_{\TV}\,dt,
\end{equation}
provided the residual forcing is integrable. The dimensionless product
$\epsilon C_{K,m}$ relates the carrier amplitude to density variation;
when $\rho>0$, its directional factor is
$\|q^\top\nabla\rho\|_1=\mathbb E_\rho|q^\top\nabla\log\rho|$.
Appendix~\ref{app:signal-strength} measures relative velocity perturbations,
and Appendix~\ref{sec:endpoint-study} reports paired endpoint errors. Numerical evaluation of the
TV bound also requires the density-weighted forcing coefficients.

\paragraph{First-order response and a controlled remainder.}

For a fixed velocity signal $s=s_{K,m}$, consider the family of fields
$u+\epsilon s$, so that $\delta u=\epsilon s$. Let $\rho^\epsilon$
be its density from the same initial distribution. The scalar signed
density $\chi(x,t)$ is the first-order response of $\rho^\epsilon$ to
the amplitude $\epsilon$: the first-order density change is
$\epsilon\chi$. It satisfies
\begin{equation}\label{eq:linearized-continuity}
\partial_t\chi+\nabla\!\cdot(u\chi)
=-\nabla\!\cdot(\rho s),\qquad \chi(\cdot,0)=0.
\end{equation}

\begin{proposition}[Explicit remainder bound]\label{prop:tv-remainder}
Under the flow assumptions of Theorem~\ref{thm:drift}, let
$\rho^\epsilon$ be generated by $u+\epsilon s$ and assume
$C_2:=\int_0^1\|\nabla\!\cdot(s(\cdot,t)\chi(\cdot,t))\|_{\TV}\,dt<\infty$.
Then the remainder in
\begin{equation}\label{eq:density-expansion}
\rho^\epsilon=\rho+\epsilon\chi+\mathcal R^\epsilon
\end{equation}
satisfies $\|\mathcal R^\epsilon(\cdot,1)\|_{\TV}\le\epsilon^2C_2$.
Writing $C_1=\|\chi(\cdot,1)\|_{\TV}$, this gives
\begin{equation}\label{eq:controlled-tv-expansion}
\|\rho^\epsilon(\cdot,1)-\rho(\cdot,1)\|_{\TV}
\le\epsilon C_1+\epsilon^2C_2.
\end{equation}
\end{proposition}

\begin{proof}
Substitute \eqref{eq:linearized-continuity} into the exact difference
equation used above. The remainder satisfies
\[
\partial_t\mathcal R^\epsilon+
\nabla\!\cdot((u+\epsilon s)\mathcal R^\epsilon)
=-\epsilon^2\nabla\!\cdot(s\chi),\qquad
\mathcal R^\epsilon(\cdot,0)=0.
\]
Transport by the perturbed flow and the TV triangle inequality yield
$\|\mathcal R^\epsilon(\cdot,1)\|_{\TV}\le\epsilon^2C_2$.
Equation~\eqref{eq:controlled-tv-expansion} follows from
\eqref{eq:density-expansion}.
\end{proof}

The following corollary characterizes the cancellation in $\chi$.

\begin{corollary}[Effect of zero time average]
\label{cor:zero-mean-gap}
Under the assumptions of Proposition~\ref{prop:tv-remainder}, let
$G_q(r)=\mathcal T_{1\leftarrow r}[q^\top\nabla\rho(\cdot,r)]$.
If $\int_0^1\alpha(r)\,dr=0$, then for any signed measure $\bar G$,
\begin{equation}
\chi(\cdot,1)
=
-\int_0^1 \alpha(r)\,\bigl(G_q(r)-\bar G\bigr)\,dr .
\label{eq:zero-mean-residual}
\end{equation}
Consequently, the zero-mean carrier cancels the time-invariant component
of the first-order forcing. If, in addition, $G_q$ is $L_q$-Lipschitz in
total variation, then choosing $\bar G=G_q(1/2)$ and
$\alpha(r)=\sin(2\pi r)$ gives
\begin{equation}
\|\chi(\cdot,1)\|_{\TV}
\le
\frac{L_q}{2\pi}.
\label{eq:zero-mean-lipschitz}
\end{equation}
The terminal density drift therefore satisfies
\begin{equation}\label{eq:zero-mean-tv-bound}
\|\rho^\epsilon(\cdot,1)-\rho(\cdot,1)\|_{\TV}
\le\frac{\epsilon L_q}{2\pi}+\epsilon^2C_2.
\end{equation}
\end{corollary}

\begin{proof}
The flow representation of Equation~\eqref{eq:linearized-continuity}
gives the first-order density response at time $1$ as
\begin{equation}
\chi(\cdot,1)
=
-\int_0^1 \alpha(r)\,
G_q(r)\,dr .
\end{equation}
Because $\int_0^1\alpha(r)\,dr=0$,
$\int_0^1\alpha(r)\bar G\,dr=0$ for any $\bar G$, which proves
\eqref{eq:zero-mean-residual}. If
$\|G_q(r)-G_q(s)\|_{\TV}\le L_q|r-s|$, then
\[
\|\chi(\cdot,1)\|_{\TV}
\le
L_q\int_0^1 |\sin(2\pi r)|\,|r-1/2|\,dr
= \frac{L_q}{2\pi},
\]
which proves \eqref{eq:zero-mean-lipschitz}. Combining this estimate with
Proposition~\ref{prop:tv-remainder} gives \eqref{eq:zero-mean-tv-bound}.
\end{proof}

\subsection{An Exactly Endpoint-Preserving Reference Field}
\label{sec:endpoint-reference}

The following analytical reference translates the clean trajectory by a
path that vanishes at both endpoints. It illustrates exact preservation
under exact integration. The learned additive-sine construction instead
uses the unshifted clean field and is evaluated separately below. Write
$q=Pc_m$.

\begin{proposition}[A field with unchanged final samples]
\label{prop:endpoint-reference}
Let the clean ODE $\dot x=u(x,t)$ have unique solutions, and define
\[
\delta(t)=\frac{a_{\mathrm{nom}}}{2\pi}(1-\cos(2\pi t))q,\qquad
v_{\mathrm{ep}}(y,t)=u(y-\delta(t),t)+\dot\delta(t).
\]
For every initial state with an existing clean solution,
$y(t)=x(t)+\delta(t)$ solves $\dot y=v_{\mathrm{ep}}(y,t)$ from the
same initial state. In particular, $y(1)=x(1)$.
\end{proposition}
\begin{proof}
Differentiate $y=x+\delta$ to obtain
$\dot y=u(x,t)+\dot\delta=v_{\mathrm{ep}}(y,t)$.
Both $\delta(0)$ and $\delta(1)$ vanish; uniqueness identifies the
corresponding solution and proves endpoint equality.
\end{proof}

\paragraph{Euler discretization of the endpoint identity.}
Let $t_n=nh$, $Jh=1$, and apply explicit Euler to $u$ and
$v_{\mathrm{ep}}$ from the same initial state. Define
$e_n=y_n-x_n-\delta(t_n)$. Direct subtraction of the two updates gives
\begin{equation}\label{eq:loop-euler-defect}
\begin{aligned}
e_{n+1}={}&e_n+h\bigl[u(x_n+e_n,t_n)-u(x_n,t_n)\bigr]+d_n,\\
d_n={}&h\dot\delta(t_n)-[\delta(t_{n+1})-\delta(t_n)].
\end{aligned}
\end{equation}
Euler approximates the increment $\delta(t_{n+1})-\delta(t_n)$ by
$h\dot\delta(t_n)$, producing the local defects $d_n$. Under smoothness,
$d_n=O(h^2)$ and their transported accumulation is $O(h)$. On a uniform
sine grid, the local defects sum to zero. The state-dependent term
transports them with different gains. Their
terminal accumulation produces the numerical endpoint discrepancy, even
for $v_{\mathrm{ep}}$.

\paragraph{Numerical evaluation of the prescribed fields.}
To test the endpoint identity before neural approximation, we integrate
$u(x,t)+a_{\mathrm{nom}}\sin(2\pi t)q$ and
$v_{\mathrm{ep}}(x,t)=u(x-\delta(t),t)+\dot\delta(t)$ directly, evaluating
the frozen clean network $u$ at each solver step. No watermarked network
is trained in this experiment. For each of 160 shared initial states,
we compare the resulting endpoint with the clean Euler endpoint from
that state. With $a_{\mathrm{nom}}=1.5$ and 100 Euler steps, the
per-coordinate RMS difference before RGB conversion is $0.002312$ for
additive sine and $4.699\times10^{-5}$ for the balanced field.
Thus the balanced construction yields substantially closer numerical
endpoints when its prescribed field is evaluated directly.
In this control, clipping is applied to the translated coordinate
$y_n-\delta(t_n)$ for the balanced field and to the unshifted state
for additive sine.
The learned-model comparison below instead integrates standalone networks
with clipping applied to their states.

\paragraph{Sources of output error.}
Write $X_1$ for the exact clean endpoint, $Y^*_1$ for the exact endpoint of
the ideal watermarked field, and $\widehat Y_1$ for that of its learned
approximation. Superscript $h$ denotes numerical sampling. Then
\begin{equation}\label{eq:endpoint-error-sources}
\begin{aligned}
\widehat Y^h_1-X^h_1
={}&\underbrace{Y^*_1-X_1}_{\text{ideal-field displacement}}
 +\underbrace{\widehat Y_1-Y^*_1}_{\text{fitting error}}\\
 &+\underbrace{(\widehat Y^h_1-\widehat Y_1)
                    -(X^h_1-X_1)}_{\text{difference of solver errors}}.
\end{aligned}
\end{equation}
For $v_{\mathrm{ep}}$, the ideal-field displacement vanishes. For the
additive-sine field $u(x,t)+\epsilon\sin(2\pi t)q$, it depends on the
interaction between the carrier and the clean dynamics. Fitting error
and solver errors contribute to the measured displacement in both cases.

\subsection{Stability of Paired Outputs}
We apply Gr\"onwall's inequality to trajectories sharing the same initial
state, relating velocity error to paired endpoint error. Related flow-matching error estimates are given by
\citet[Theorem~1]{benton2024error}.

\begin{proposition}[Propagation of velocity error]
\label{prop:gronwall-wasserstein}
Let $X_t$ and $Y_t$ solve $\dot X_t=u(X_t,t)$ and
$\dot Y_t=v(Y_t,t)$ from the same random initial state $Z$.
Assume the solutions exist uniquely, have finite second moments, and
$\|u(x,t)-u(y,t)\|_2\le \ell(t)\|x-y\|_2$ on the relevant domain,
where $\ell\ge0$ is integrable. Write $r(x,t)=v(x,t)-u(x,t)$ and assume
$\int_0^1\|r(Y_t,t)\|_{L^2}\,dt<\infty$, with
$\|A\|_{L^2}=(\mathbb E\|A\|_2^2)^{1/2}$.
Then
\begin{equation}\label{eq:gronwall-wasserstein}
\|Y_1-X_1\|_{L^2}
\le \int_0^1
\exp\!\left(\int_t^1 \ell(s)\,ds\right)
\|r(Y_t,t)\|_{L^2}\,dt .
\end{equation}
The same statement applies conditionally on a fixed generation label.
\end{proposition}
\begin{proof}
Couple the two trajectories using the same $Z$. For almost every trajectory,
the integral equations and the spatial Lipschitz bound give
\[
\|Y_t-X_t\|_2
\le\int_0^t \ell(s)\|Y_s-X_s\|_2\,ds
   +\int_0^t\|r(Y_s,s)\|_2\,ds.
\]
The inhomogeneous Gr\"onwall inequality therefore yields
\[
\|Y_1-X_1\|_2
\le\int_0^1
\exp\!\left(\int_t^1\ell(s)\,ds\right)\|r(Y_t,t)\|_2\,dt.
\]
Minkowski's integral inequality bounds the $L^2$ norm of the right-hand side
by the integral in~\eqref{eq:gronwall-wasserstein}, proving the bound.
\end{proof}
For the endpoint laws $\mu_1=\operatorname{Law}(X_1)$ and
$\nu_1=\operatorname{Law}(Y_1)$, pairing outputs generated from the same
initial state is a coupling of these laws,
so $W_2(\mu_1,\nu_1)\le\|Y_1-X_1\|_{L^2}$.

\paragraph{Application to the proposed signal.}
If the learned field is
$v_\theta=u+\epsilon\sin(2\pi t)q+e_\theta$, then
\begin{equation}\label{eq:additive-gronwall}
\|Y_1-X_1\|_{L^2}
\le\int_0^1 A(t)
\left[\epsilon|\sin(2\pi t)|\,\|q\|_2
      +\|e_\theta(Y_t,t)\|_{L^2}\right]dt,
\qquad A(t)=\exp\!\left(\int_t^1\ell(s)\,ds\right).
\end{equation}
For a uniform Lipschitz constant $\ell_{\max}$ and unit $q$, a simpler, looser bound is
$e^{\ell_{\max}}[2\epsilon/\pi+\int_0^1\|e_\theta(Y_t,t)\|_{L^2}dt]$.
This bound accumulates the absolute forcing magnitude.

\paragraph{Numerical integration and reported RMS.}
Let $X^h_1,Y^h_1$ be numerical final samples and define their root-mean-square
errors relative to the corresponding exact ODE solutions as
$\eta_u(h)=\|X^h_1-X_1\|_{L^2}$ and
$\eta_v(h)=\|Y^h_1-Y_1\|_{L^2}$. The triangle inequality gives
\begin{equation}\label{eq:numerical-gronwall}
\|Y^h_1-X^h_1\|_{L^2}
\le\eta_u(h)+B+\eta_v(h),
\end{equation}
where $B$ is the right-hand side of Equation~\eqref{eq:additive-gronwall}.
For $D$ coordinates, the reported per-coordinate RMS is
$R=(\mathbb E\|Y_1-X_1\|_2^2/D)^{1/2}$, corresponding to paired $L^2$
error $\sqrt D\,R$. The RGB conversion
$C(x)=(\operatorname{clip}_{[-1,1]}x+1)/2$ is $1/2$-Lipschitz, so RGB
errors are at most half the corresponding state-coordinate errors.

\paragraph{Normalized image distortion energy.}
Let $\{Y_i^0\}_{i=1}^{n_0}\subset[0,1]^D$ be a study's fixed clean output
reference bank. Define
\begin{equation}\label{eq:relative-energy}
S_0^2=\frac1{n_0D}\sum_{i=1}^{n_0}\|Y_i^0\|_2^2,
\qquad
E_{\mathrm{rel}}=\frac{R_{\mathrm{RGB}}^2}{S_0^2}
=R_{\mathrm{rel}}^2,
\end{equation}
where $R_{\mathrm{rel}}=R_{\mathrm{RGB}}/S_0$ is relative RMS error.
For matched input banks, $E_{\mathrm{rel}}$ equals
$\sum_i\|Y_i^w-Y_i^0\|_2^2/\sum_i\|Y_i^0\|_2^2$.
Tables report $100E_{\mathrm{rel}}$ as NMSE (\%). The same uncentered
clean-output energy normalizes endpoint, mean paired, and solver errors
within each study. Averages and standard deviations are computed from
per-model energy ratios.

\subsection{Matched Comparison of Velocity-Field Variants}
\label{sec:matched-constructions}

We compare additive sine, balanced sine, and the time-independent signal
using the same CIFAR-10 clean backbone, shared H32 spatial key and training
seed, and a common fitting recipe. We set $a_{\mathrm{nom}}=1.5$ for the
sinusoidal targets and $a_{\mathrm{nom}}/\sqrt2$ for the time-independent
target, giving each prescribed velocity signal energy $a_{\mathrm{nom}}^2/2$.
The balanced target additionally changes the clean field's spatial
argument; the energy match therefore applies to its prescribed additive
velocity component, not to its full difference from the clean field.

Each standalone model is fitted by updating all parameters for 1,000 AdamW steps at
learning rate $10^{-5}$, batch size 128, and zero weight decay. The loss
is mean-squared velocity error over all output coordinates plus a keyed-subspace error
penalty of weight 0.1. Training states mix
clean-model trajectories and noise--data interpolants in equal proportions;
the balanced variant translates these states by $\delta(t)$ when fitting
its translated-field target. The primary modification study instead uses
learning rate $10^{-4}$ and the clean-reference weighted regression in
Equation~\eqref{eq:loss} with $\lambda_{\rm corr}=0$.

All three models use 100 FP32 Euler steps and identical evaluation inputs.
Table~\ref{tab:matched-constructions} reports uncorrected output error on
256 held-out inputs, normalized by their shared clean-output energy.
The additive and balanced models both achieve complete message recovery
in the saved velocity-query evaluation. The balanced target provides no observed endpoint-fidelity advantage under
the matched conditions. Its translated training states
are part of the evaluated fitting procedure.

\begin{table}[!htb]\centering\small
\caption{Matched learned velocity-field variants on CIFAR-10. NMSE uses
one shared clean-output reference; RMS is in RGB coordinates.}
\label{tab:matched-constructions}
\begin{tabular}{lrr}\toprule
Construction & Image NMSE (\%) & RGB RMS \\
\midrule
Additive sine & 0.04408 & 0.01089 \\
Balanced sine & 0.05109 & 0.01172 \\
Time-independent & 0.05724 & 0.01241 \\
\bottomrule\end{tabular}
\end{table}

\paragraph{Endpoint effect of temporal modulation.}
To test whether sinusoidal modulation reduces the signal's endpoint
effect relative to a time-independent signal, we compare the fitted
models in Table~\ref{tab:matched-constructions}. For each model, integrate
$v_\theta$ and $v_\theta-f$ from the same initial state, where
$f(t)=1.5\sin(2\pi t)q$ for sine and $f(t)=(1.5/\sqrt2)q$ for the
time-independent signal. The model weights and remaining velocity field
are fixed. Define the endpoint response relative to signal energy by
\begin{equation}\label{eq:carrier-endpoint-response}
R_{\mathrm{ep}}(f)
=\frac{\mathbb E\|Y_1^{v_\theta}-Y_1^{v_\theta-f}\|_2^2}
{\int_0^1\|f(t)\|_2^2\,dt}.
\end{equation}
Both quantities use the model's full state coordinates before RGB
conversion. The denominator is $1.125$ for both signals. On the same
256 inputs with 100 Euler steps, $100R_{\mathrm{ep}}$ is $1.414\%$ for
sine and $14.081\%$ for the time-independent signal. This lower endpoint
response supports the sinusoidal carrier in the matched fitting
experiment. The two separately fitted fields can have different
approximation residuals.

\subsection{Subtracting the Mean Output Error}
\label{sec:output-calibration}

We define output bias as the mean paired endpoint error
$\mathbb E[Y^{\mathrm{watermarked}}-Y^{\mathrm{clean}}]$, in RGB coordinates.
A provider estimates this mean from matching clean and watermarked
outputs on a calibration set:
\begin{equation}\label{eq:output-calibration}
 b=\frac1{n_{\mathrm{cal}}}\sum_i
 (Y_i^{\mathrm{watermarked}}-Y_i^{\mathrm{clean}}),\qquad
 Y^{\mathrm{cal}}=\operatorname{clip}_{[0,1]}(Y^{\mathrm{watermarked}}-b).
\end{equation}
After calibration, the provider applies the fixed $b$ to generated outputs.
The velocity field is unchanged. Appendix~\ref{app:corrected-transfer}
evaluates sample-only transfer after this correction.

\subsection{Measured Endpoint Distortion}
\label{sec:endpoint-study}\label{sec:output-experiments}

\paragraph{Shared models and inputs.}
We reuse the five full-weight H32 checkpoints from
Section~\ref{sec:modification} and their frozen clean reference. The test
bank contains 512 initial noises and labels, shared across checkpoints,
with noise seed 493042. Sampling uses FP32 Euler steps and internal state
clipping to $[-5,5]$. All 512 test inputs use 100 Euler steps. All model weights remain fixed.

\paragraph{Total output distortion and mean correction.}
We compare marked and clean outputs after final RGB conversion to
$[0,1]$. Each checkpoint has a fixed RGB offset estimated by
Equation~\eqref{eq:output-calibration} from 256 separate paired
calibration outputs, with noise seed 493041. For test errors
$\Delta_i=Y_i^{\mathrm{watermarked}}-Y_i^{\mathrm{clean}}$, we compute
\begin{equation}\label{eq:direct-rms}
\widehat R=\sqrt{\frac{1}{nD}\sum_{i=1}^{n}\|\Delta_i\|_2^2},
\qquad
\widehat R_{\mathrm{mean}}
=\frac{1}{\sqrt D}\left\|\frac1n\sum_{i=1}^{n}\Delta_i\right\|_2,
\end{equation}
with $n=512$ and $D=3072$. Corrected metrics replace the watermarked
output by $\operatorname{clip}_{[0,1]}(Y_i^{\mathrm{watermarked}}-b)$ using
the calibration offset. Table~\ref{tab:direct-mean-per-message} gives the
individual measurements. Across messages, image NMSE falls from
$0.162\pm0.018\%$ to
$0.111\pm0.009\%$ after correction, a 31.9\%
reduction in mean error energy. The corresponding relative RMS
errors are 4.03\% and 3.33\%, respectively.

\begin{table}[!htb]\centering\small
\caption{Endpoint and mean-error NMSE (\%) before and after output
correction, normalized by the same clean-output energy.}
\label{tab:direct-mean-per-message}
\begin{tabular}{lrrrr}\toprule
Message & Original & Corrected & Original mean & Corrected mean \\
\midrule
1 & 0.19049 & 0.12659 & 0.0639 & 0.0006 \\
2 & 0.14352 & 0.09886 & 0.0446 & 0.0005 \\
3 & 0.14304 & 0.10528 & 0.0382 & 0.0007 \\
4 & 0.17009 & 0.11308 & 0.0569 & 0.0005 \\
5 & 0.16524 & 0.10981 & 0.0557 & 0.0007 \\
\bottomrule\end{tabular}
\end{table}

\paragraph{Matched numerical reference.}
We also generate clean RK4 outputs with 25 and 100 steps on the same 512
test inputs. These samplers use FP32 network evaluations and states, the
same post-step state clipping, and the same final RGB conversion as Euler.
Table~\ref{tab:solver-reference} compares the clean solver discrepancy with
the endpoint distortion of the current watermarked checkpoints.

\begin{table}[!htb]\centering\small
\caption{Numerical-solver and watermarked-model NMSE (\%) on 512 paired
inputs. Watermarked entries are means and population standard deviations
across five messages.}
\label{tab:solver-reference}
\begin{tabular}{lr}\toprule
Comparison & NMSE (\%) \\
\midrule
Clean Euler-100 vs. clean RK4-100 & 0.01322 \\
Clean RK4-25 vs. clean RK4-100 & 0.00014 \\
Watermarked Euler-100 vs. clean Euler-100 & $0.162\pm0.018$ \\
Corrected watermarked Euler-100 vs. clean Euler-100 & $0.111\pm0.009$ \\
\bottomrule\end{tabular}
\end{table}

\section{Proof of Theorem~\ref{thm:capacity}}
\label{app:capacity_proof}

We derive a Gaussian model for the averaged projected velocities and use
it to bound the information available for message recovery.

\subsection{Averaged Scores for a General Carrier}

Use a carrier $\alpha(t)$ and the additive model
$v_{K,m}(x,t)=v_{\mathrm{bg}}(x,t)+\epsilon\alpha(t)Pc_m$, with
$P^\top P=I_k$. For a fixed key, the background is shared across messages.
Averaging $N$ queries gives the exact decomposition
\begin{equation}\label{eq:finite-query-channel}
\begin{aligned}
\hat s_N&=\epsilon\widehat\kappa_\alpha c_m+b_K+\xi_N,\\
\widehat\kappa_\alpha&=\frac1N\sum_{i=1}^N\alpha(t_i)^2,\qquad
\xi_N=\frac1N\sum_{i=1}^N(z_i-b_K),\\
z_i&=\alpha(t_i)P^\top v_{\mathrm{bg}}(x_i,t_i),\qquad
b_K=\mathbb E[z_i].
\end{aligned}
\end{equation}

\begin{assumption}[Independent queries and finite moments]\label{ass:iid}
Queries $(x_i,t_i)$ are i.i.d.~under a distribution $p_q$ independent of
$M$, with $0<\kappa_\alpha=\mathbb E_{p_q}[\alpha(t)^2]<\infty$ and
$\mathbb E\|z_i\|_2^2<\infty$.
\end{assumption}

The law of large numbers gives
$\widehat\kappa_\alpha\to\kappa_\alpha$ almost surely. With
$\Sigma_\alpha=\operatorname{Cov}(z_i)$, the multivariate central limit
theorem gives $\sqrt N\xi_N\Rightarrow\mathcal N(0,\Sigma_\alpha)$.
The Gaussian model replaces $\widehat\kappa_\alpha$ by its expectation
$\kappa_\alpha$ and uses this limiting law for the centered background average.

\begin{proposition}[Consistency of maximum-correlation decoding]
\label{prop:decoder-consistency}
Under Assumption~\ref{ass:iid}, let $\epsilon>0$, let the codebook be a
finite set of distinct unit-norm vectors, and suppose
$\widehat b_K\to b_K$ almost surely. Then maximum-correlation decoding of
$\hat s_N-\widehat b_K$ recovers the embedded message almost surely as
$N\to\infty$.
\end{proposition}

\begin{proof}
Equation~\eqref{eq:finite-query-channel} and the law of large numbers give
$\hat s_N-\widehat b_K\to\epsilon\kappa_\alpha c_m$ almost surely.
For every $j\ne m$, unit norm and distinctness imply
$\langle c_m,c_m\rangle=1>\langle c_m,c_j\rangle$. The finite codebook
therefore has a positive correlation gap, which is preserved for all
sufficiently large $N$.
\end{proof}

\begin{assumption}[Finite-query Gaussian approximation]\label{ass:gaussian}
For channel analysis at the stated query count, use
\begin{equation}\label{eq:channel_final}
\hat s=\epsilon\kappa_\alpha c_M+b_K+\xi,
\qquad \xi\sim\mathcal N(0,\Sigma_\alpha/N),
\end{equation}
with $\xi$ independent of $M$.
\end{assumption}

Equation~\eqref{eq:finite-query-channel} includes the random variation of
$\widehat\kappa_\alpha$ across query sets; Equation~\eqref{eq:channel_final}
is the Gaussian model used in Theorem~\ref{thm:capacity}.

Writing $B(x,t)=\alpha(t)v_{\mathrm{bg}}(x,t)$ gives
$b_K=P^\top\mathbb E[B(x,t)]$ and
$\Sigma_\alpha=P^\top\operatorname{Cov}(B(x,t))P$.
\begin{assumption}[Positive-definite background covariance]\label{ass:invertible}
The projected covariance satisfies $\Sigma_\alpha\succ0$.
\end{assumption}

\subsection{Finite-Codebook Information Bound}

\begin{proof}[Proof of Theorem~\ref{thm:capacity}]
Let $g=\epsilon\kappa_\alpha$ and $Y=\hat s-b_K=gc_M+\xi$.
Translation gives $I(M;\hat s)=I(M;Y)$. For
$Q=\operatorname{Cov}(c_M)$, independence of $\xi$ and $M$ yields
$\operatorname{Cov}(Y)=g^2Q+\Sigma_\alpha/N$. Gaussian maximum
entropy~\citep{cover2006elements} gives
\begin{align}
I(M;Y)&=h(Y)-h(\xi)\\
&\le\frac12\log\frac{\det(g^2Q+\Sigma_\alpha/N)}
                         {\det(\Sigma_\alpha/N)}\\
&=\frac12\log\det\!\left(I_k+N\epsilon^2\kappa_\alpha^2
\Sigma_\alpha^{-1/2}Q\Sigma_\alpha^{-1/2}\right).
\end{align}
The bound $I(M;Y)\le H(M)$ gives the message-entropy bound. Since the
codewords have unit norm,
$\operatorname{tr}Q=1-\|\mathbb E[c_M]\|_2^2\le1$.
Maximizing the Gaussian expression over positive semidefinite covariances
of trace at most one gives the Gaussian information bound under the power constraint.
\end{proof}

\subsection{Query and Message Length Scaling}\label{app:hypercube-scaling}

Let $\nu_1,\ldots,\nu_k$ be the eigenvalues of
$\Sigma_\alpha^{-1/2}Q\Sigma_\alpha^{-1/2}$. The mutual information of the Gaussian expression is
\[
\frac12\sum_j\log(1+N\epsilon^2\kappa_\alpha^2\nu_j).
\]
Its large-$N$ coefficient of $\log N$ is $\operatorname{rank}(Q)/2$.
The finite-alphabet information is capped by $H(M)$ and approaches this
cap as distinct codewords become distinguishable.

For uniform hypercube messages, $H(M)=L\log2$ and
$Q=L^{-1}\sum_{j=1}^L e_je_j^T$. Under isotropic background covariance,
sign decoding of $Y=\hat s-b_K$ has bit-error probability
\[
p_{\mathrm{bit}}=\Phi\!\left(
-\frac{\epsilon\kappa_\alpha\sqrt N}{\sigma\sqrt L}\right).
\]
Independent Gaussian coordinates give message recovery
$(1-p_{\mathrm{bit}})^L$. For uniform-time sine,
$\kappa_\alpha=\int_0^1\sin^2(2\pi t)\,dt=1/2$, giving the experimental
specialization $p_{\mathrm{bit}}=\Phi(-\epsilon\sqrt N/(2\sigma\sqrt L))$.
\section{Evaluation and Embedding Protocols}\label{app:embedding-configurations}

This appendix first specifies the common recovery and quality evaluation,
then gives the embedding protocol for each model family. All models use the
spatial codes and synchronous decoder of Section~\ref{sec:wmrc}; only the
optimized parameters and regression objective differ across model families.

\subsection{Evaluation Protocol}\label{app:recovery-quality}

\begin{table}[!htb]\centering\normalsize
\caption{Recovery, clean-model matches to the claimed message, and quality
at $k=32$. MLP matches use fresh random keys; other rows average saved clean
decodes over code labels.}
\label{tab:all-generators}
\setlength{\tabcolsep}{3pt}
\begin{tabular}{@{}lllccc@{}}\toprule
Generator & Data & Code & Recovery (\%) & Clean label match (\%) & Quality ratio \\
\midrule
Image MLP & MNIST & O5 & $100.0{\pm}0.0$ & 3.100 & $1.037{\pm}0.006$ \\
Image MLP & MNIST & H5 & $100.0{\pm}0.0$ & 3.160 & $1.040{\pm}0.010$ \\
Image MLP & MNIST & H32 & $100.0{\pm}0.0$ & 0 (observed) & $1.041{\pm}0.022$ \\
\addlinespace[2pt]
RF-UNet & MNIST & O5 & $100.0{\pm}0.0$ & 3.115 & $0.974{\pm}0.031$ \\
RF-UNet & MNIST & H5 & $100.0{\pm}0.0$ & 3.114 & $0.983{\pm}0.030$ \\
RF-UNet & MNIST & H32 & $100.0{\pm}0.0$ & 0 (observed) & $1.014{\pm}0.012$ \\
\addlinespace[2pt]
RF-UNet & CIFAR-10 & O5 & $100.0{\pm}0.0$ & 3.132 & $1.010{\pm}0.018$ \\
RF-UNet & CIFAR-10 & H5 & $100.0{\pm}0.0$ & 3.139 & $1.026{\pm}0.034$ \\
RF-UNet & CIFAR-10 & H32 & $100.0{\pm}0.0$ & 0 (observed) & $1.032{\pm}0.021$ \\
\addlinespace[2pt]
RF-UNet & CelebA-64 & O5 & $100.0{\pm}0.0$ & 3.131 & $1.073{\pm}0.007$ \\
RF-UNet & CelebA-64 & H5 & $100.0{\pm}0.0$ & 3.118 & $1.093{\pm}0.015$ \\
RF-UNet & CelebA-64 & H32 & $100.0{\pm}0.0$ & 0 (observed) & $1.036{\pm}0.025$ \\
\addlinespace[2pt]
SD3.5 DiT & Image latents & O5 & $100.0{\pm}0.0$ & 3.086 & $1.035{\pm}0.045$ \\
SD3.5 DiT & Image latents & H5 & $100.0{\pm}0.0$ & 3.142 & $1.024{\pm}0.017$ \\
SD3.5 DiT & Image latents & H32 & $100.0{\pm}0.0$ & 0 (observed) & $0.992{\pm}0.019$ \\
\addlinespace[2pt]
Audio MLP & FMA & O5 & $100.0{\pm}0.0$ & 3.106 & $1.005{\pm}0.002$ \\
Audio MLP & FMA & H32 & $100.0{\pm}0.0$ & 0 (observed) & $1.003{\pm}0.002$ \\
\addlinespace[2pt]
Audio Transformer & FMA & O5 & $100.0{\pm}0.0$ & 3.140 & $0.994{\pm}0.004$ \\
Audio Transformer & FMA & H32 & $100.0{\pm}0.0$ & 0 (observed) & $0.999{\pm}0.004$ \\
\addlinespace[2pt]
Stable Audio & FMA & O5 & $100.0{\pm}0.0$ & 3.142 & $1.006{\pm}0.028$ \\
Stable Audio & FMA & H32 & $100.0{\pm}0.0$ & 0 (observed) & $0.933{\pm}0.005$ \\
\bottomrule\end{tabular}
\end{table}

For Tables~\ref{tab:main} and~\ref{tab:all-generators}, recovery uses five target messages and 20 trials per message, except
CIFAR-10 H32, which uses 40 trials. Image MLP, RF-UNet, and audio trials
use 4096 velocity queries; SD3.5 trials use 256. Image MLP queries are
pure noise. RF-UNet, SD3.5, and audio queries use the interpolation
distributions associated with their data or latent pools.

Image MLP and RF-UNet quality use 500 samples with the same Inception
feature extractor and preprocessing. SD3.5 uses the standard
descending-noise sampler, with 500 samples for H32 and 30 for O5 and H5.
Scratch audio uses 500 clips; Stable Audio uses 500 clips for H32 and 32 for O5.
The high-resolution SD3.5 figure uses 40 sampling steps and guidance
scale 4.5; its matched clean and watermarked seeds are fixed before generation.

Figures~\ref{fig:cifar10-paired} and~\ref{fig:sd35-quality} show paired
comparisons for CIFAR-10 and SD3.5.

\begin{figure}[!htb]
\centering
\includegraphics[width=0.9\textwidth]{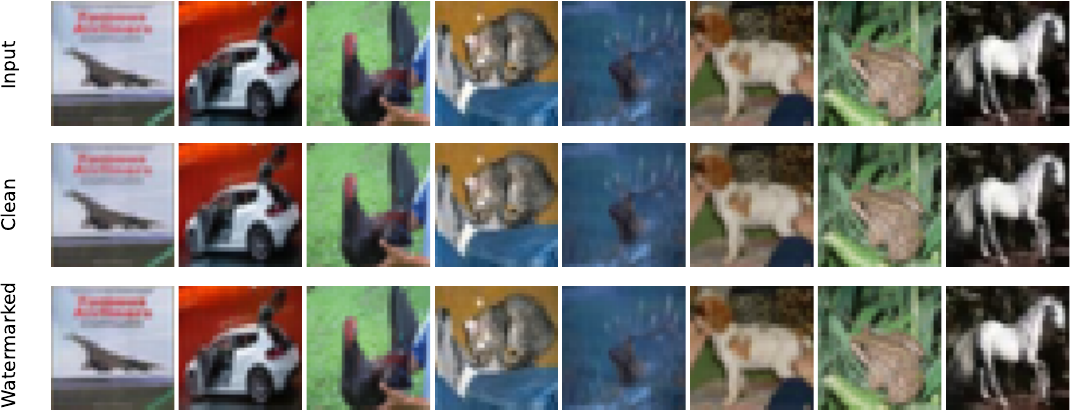}
\caption{CIFAR-10 test images (top) reconstructed through the clean (middle)
and primary full-weight H32 flows (bottom). Each input is integrated backward
through the clean field and then forward through either field using 100 Euler
steps in each direction.}
\label{fig:cifar10-paired}
\end{figure}

\begin{figure}[!htb]
\centering
\includegraphics[width=0.9\textwidth]{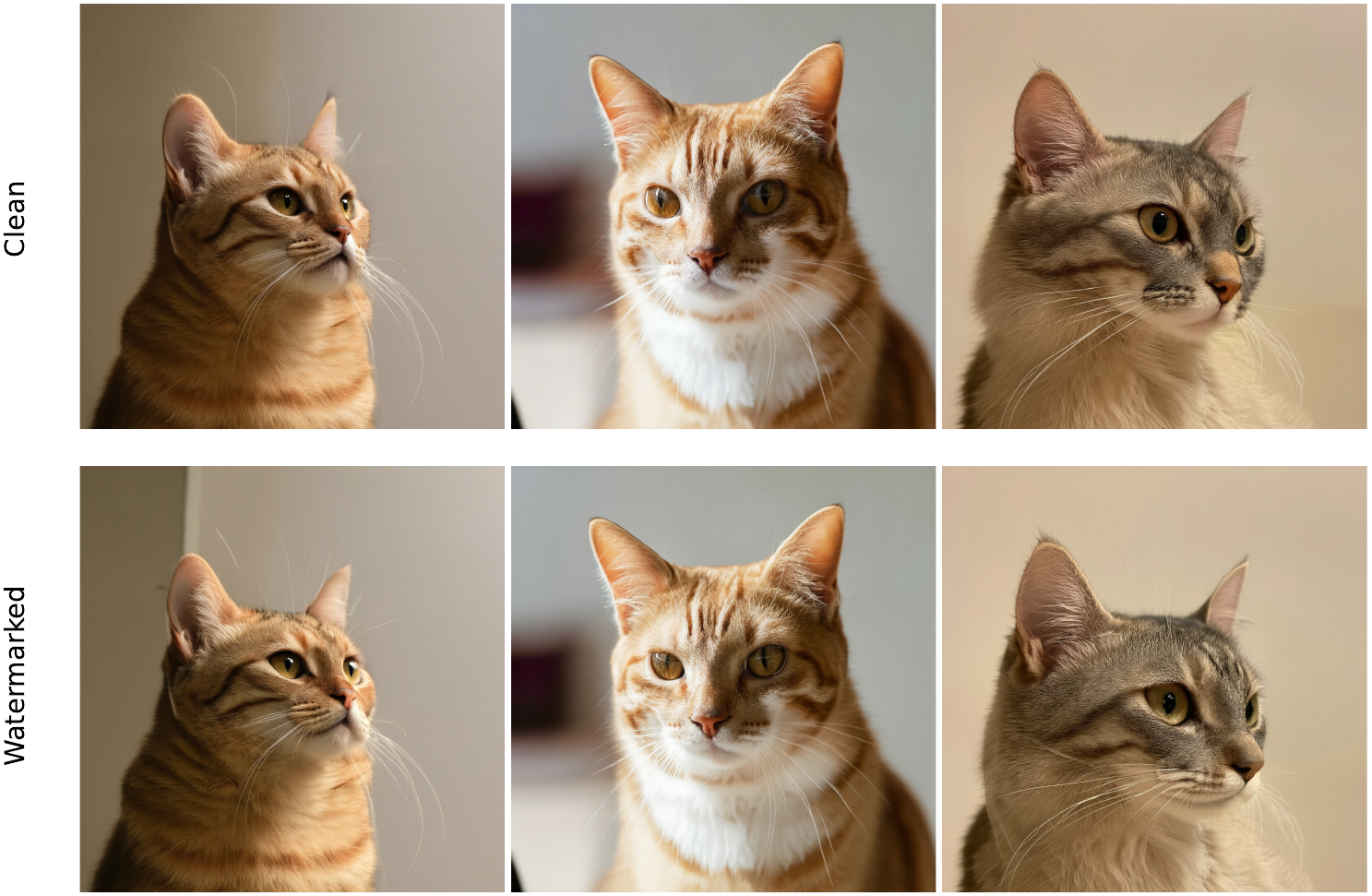}
\caption{Paired SD3.5 Medium samples from clean (top) and
hypercube-watermarked (bottom) models, using identical noise seeds and
the watermark-training prompt, ``a photo of a cat.'' The figure compares
visual quality at $1024\times1024$.}
\label{fig:sd35-quality}
\end{figure}

\subsection{Embedding Protocols by Model Family}

All protocols fit the same carrier and spatial code defined in
Section~\ref{sec:concrete}. Pretrained-model experiments use the quadratic
terms of Equation~\eqref{eq:loss}; the model-specific paragraphs below
state the regression reference, optimized parameters, and correlation
weight. When a correlation reward is used, its linear term is implemented as
\begin{equation}\label{eq:bounded-correlation-variant}
-\lambda_{\rm corr}\tanh\!\left(
\frac{4\Gamma(\theta)}{a_{\mathrm{nom}}}\right).
\end{equation}
For the uniform-time sinusoidal carrier, this is a bounded monotone
function of the same demodulated correlation. The first-variation argument
of Proposition~\ref{prop:optimal-predictor} therefore gives the same
carrier and code direction; only the fitted amplitude differs.

\subsubsection{Scratch-Trained Models}

The image MLPs are trained from scratch for 115,000 updates
at batch size 512 on 12,000 images, followed by 50 updates of velocity
distillation. The scratch audio MLP and Transformer use a fixed
Music2Latent codec~\citep{music2latent}, 2,000 training clips, and 500
held-out clips. Clean training uses 5,000 updates; embedding uses 1,000
updates with batch size 32 and initial learning rate $10^{-4}$.

\subsubsection{Pretrained RF-UNet Models}\label{app:rf-unet-protocol}

For the cross-model RF-UNet study, embedding optimizes rank-16 LoRA
adapters on convolutional and linear layers, together with the
time-embedding and output-projection layers, including attention and final
outputs. The remaining pretrained
parameters are frozen. The CIFAR-10 payload-scaling study uses the same
trainable parameter set. Full-weight embedding for the modification and
four-method studies is specified in Appendix~\ref{app:full-weight-studies}.
The cross-model RF-UNet checkpoints use velocity regression to the marked
target, a keyed-subspace MSE term of weight 10, and the bounded
correlation reward in Equation~\eqref{eq:bounded-correlation-variant} with
$\lambda_{\rm corr}=1$. CelebA-64 H32 uses clean-reference regression and
the keyed-subspace term for 1,000 updates, with $\lambda_{\rm corr}=0$.
The primary full-weight CIFAR-10 checkpoints also use
Equation~\eqref{eq:loss} with $\lambda_{\rm corr}=0$.

\subsubsection{SD3.5}

SD3.5 H32 uses 700 optimizer updates with rank-16 LoRA, a
keyed-subspace MSE term, and Equation~\eqref{eq:bounded-correlation-variant}. The final 150 updates adjust
the target using an estimate of the clean projected response, with nominal
strength 3.0 and $\lambda_{\rm corr}=0.5$ in the final objective.
Verification uses raw watermarked velocities and the saved key.

\subsubsection{Stable Audio}\label{app:audio-results}

Audio experiments use six-second clips from FMA-small~\citep{fmadataset},
with training and held-out recordings following the official split.

Stable Audio 3 Small Music Base is the rectified-flow
checkpoint~\citep{stableaudio3base}. Its own fixed codec produces
$256\times66$ latents per clip; the projection acts on all 16,896
flattened coordinates. The experiment uses 128 training and 32 held-out
recordings, the fixed prompt \textit{An instrumental music recording},
and rank-16, alpha-16 LoRA on linear layers (10,277,184 trainable
parameters). The backbone's descending noise-time coordinate is
converted to forward flow time by $v(x,t)=-u(x,1-t)$.
Embedding uses batch size one, initial learning rate $10^{-4}$,
AdamW, cosine decay, and nominal strength 1.5.

Stable Audio O5 uses 100 updates per message and H32 uses 400 total
updates.

The clean mean projected response $b_K$ is estimated from 4096
training-pool queries and subtracted in the training target. With
$\alpha(t)=\sin(2\pi t)$ and nominal strength $a_{\mathrm{nom}}=1.5$, the projected
residual target is $\alpha(t)(a_{\mathrm{nom}} c_m-2b_K)$, using
$\mathbb E[\alpha(t)^2]=1/2$.

Stable Audio uses a keyed-subspace MSE weight of 10 and
Equation~\eqref{eq:bounded-correlation-variant} with
$\lambda_{\rm corr}=1$.
The final 100 H32 updates use the target above.
Clean-model information is used only for embedding; recovery
uses raw watermarked velocities and the key. Bias estimation and recovery
use disjoint training and held-out pools.

The codec encodes padded 44.1\,kHz inputs. Decoded clips are truncated
to six seconds and averaged to mono for quality evaluation. Audio quality
features use the fixed \texttt{laion/clap-htsat-unfused}
checkpoint.\footnote{\url{https://huggingface.co/laion/clap-htsat-unfused}}
Generation
uses 100 forward-Euler steps without classifier-free guidance
amplification.

\subsection{Clean-Model False-Positive Evaluation}\label{app:fpr}

Detection accepts a claimed message if the decoded message equals it.
Corollary~\ref{cor:clean-fpr} gives FPR $2^{-L}$ under random-key label
symmetry: $3.125\%$ for five-bit codes and
$2^{-32}\approx2\times10^{-10}$ for H32. We test the corollary on the image
MLP using 10{,}000 fresh keys and 4096 fresh queries per key. The observed
five-bit rates are $3.10\%$ for O5 and $3.16\%$ for H5. No H32 false
positive is observed, giving a one-sided 95\% upper bound of $0.030\%$.
The other rows of Table~\ref{tab:all-generators} report the corresponding
code-label averages for their saved clean decodes.

\section{Coding, Signal Strength, and Carrier Learning}\label{app:tables}

This appendix reports the explicit codebooks, query and payload scaling,
measured signal strength, and the controlled temporal-carrier comparison,
in that order.

\label{app:payload}

\subsection{Explicit Random Codebooks}

The explicit-codebook sweep in Table~\ref{tab:app-capacity} fixes
$k=32$ and varies $L$ from 5 to 20. Each message has an independently
sampled vector $g_m\sim\mathcal N(0,I_k)$ normalized to
$c_m=g_m/\|g_m\|_2$. Their pairwise angles depend on the sampled codebook.
This construction stores and searches $2^L$ independently sampled codewords.
Hypercube encoding represents the message through $L$ signed coordinates.

\begin{table}[!htb]
\centering
\caption{Explicit random-codebook message length scaling at $k=32$.
Recovery uses 40 trials of 4096 queries per message; FID uses
a real-data reference.}
\label{tab:app-capacity}
\small
\begin{tabular}{ccc}
\toprule
Bits & Message recovery & FID ratio \\
\midrule
5 & $100.0{\pm}0.0$\% & $1.043{\pm}0.020$ \\
8 & $100.0{\pm}0.0$\% & $1.031{\pm}0.016$ \\
12 & $100.0{\pm}0.0$\% & $1.030{\pm}0.012$ \\
16 & $100.0{\pm}0.0$\% & $1.036{\pm}0.018$ \\
18 & $100.0{\pm}0.0$\% & $1.015{\pm}0.016$ \\
20 & $100.0{\pm}0.0$\% & $0.988{\pm}0.013$ \\
\bottomrule
\end{tabular}
\end{table}

The difference between the target-code score and the largest competing-code
score decreases from $0.478$ to $0.197$ across this sweep; all evaluated
messages remain recoverable.

\subsection{Query Count and Learned Bit Margins}\label{app:query-scaling}

Figure~\ref{fig:payload-queries} separates query variation from the
learned response. The signed score of bit $j$ is
$(2m_j-1)\langle\hat{s},e_j\rangle$; a negative score gives an
incorrect bit. Increasing query count brings the 64-bit watermark
close to complete recovery, consistent with reduced estimation noise.
At larger message lengths, some bit means remain incorrectly signed as
averaging concentrates the scores around their learned values. These
systematic errors contribute to the observed recovery loss.

\begin{figure}[!htb]
\centering
\includegraphics[width=\textwidth]{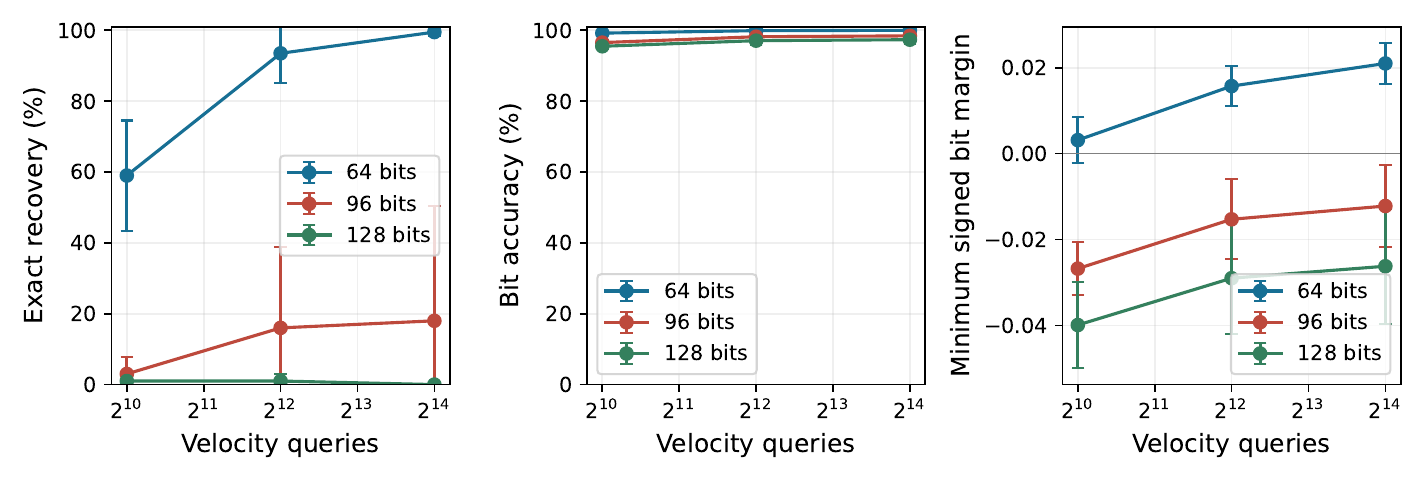}
\caption{Query scaling for 64-, 96-, and 128-bit hypercube watermarks
at fixed $k=128$ and $a_{\mathrm{nom}}=1.5$. Curves show message recovery,
bit accuracy, and the minimum signed bit score; error bars indicate
population standard deviation. Each
query count uses 40 trials per message, with shared query prefixes
within each trial.}
\label{fig:payload-queries}
\end{figure}

For the $k=128$ runs, we compare the trained watermarked model $v_w$
with its shared base model $v_0$ on the same 4,096 held-out CIFAR-10
noise--image interpolation points. Let $q=P c_m$ and
$\alpha_i=\sin(2\pi t_i)$. We report the RMS model-change ratio
$r_{\rm change}$ and its component correlated with $\alpha(t)$ $r_{\rm carr}$:
\begin{equation}
\label{eq:measured-strength}
r_{\rm change}=
\sqrt{\frac{\sum_i\|v_w(x_i,t_i)-v_0(x_i,t_i)\|_2^2}
{\sum_i\|v_0(x_i,t_i)\|_2^2}},\qquad
r_{\rm carr}=
\frac{\left|\sum_i\alpha_i\langle v_w(x_i,t_i)-v_0(x_i,t_i),q\rangle\right|}
{\sqrt{\sum_i\alpha_i^2}\sqrt{\sum_i\|v_0(x_i,t_i)\|_2^2}}.
\end{equation}
The fitted carrier amplitude is
$\widehat\epsilon=\sum_i\alpha_i\langle v_w(x_i,t_i)-v_0(x_i,t_i),q\rangle/
\sum_i\alpha_i^2$, whereas $a_{\mathrm{nom}}$ is the target-setting hyperparameter.
Across five messages, $r_{\rm carr}$ is $1.873\pm0.023\%$ at 64 bits
and $1.897\pm0.018\%$ at 128 bits, while $r_{\rm change}$ is
$12.20\pm0.21\%$ and $11.99\pm0.29\%$, respectively. Relative to the base velocity's
RMS along the keyed direction $q$, the signal component is
$1.066\pm0.019$ and $1.044\pm0.044$ times as large.

\subsection{Training Parameters and Measured Signal Strength}
\label{app:signal-strength}

Table~\ref{tab:cifar-objective-ablation} compares candidate regression
losses at
$a_{\mathrm{nom}}=1.5$ using three H32 keys, shared initialization and
training draws, 4,096 held-out queries, 20 verification trials per key,
and 128 paired endpoints. The flow-matching target is
$x_1-x_0+a_{\mathrm{nom}}\alpha(t)Pc_m$; the clean-model target is
$u(x_t,t)+a_{\mathrm{nom}}\alpha(t)Pc_m$.

\begin{table}[!htb]\centering\small
\caption{CIFAR-10 objective comparison. \emph{Velocity MSE} is
$D^{-1}\|e_\theta\|_2^2$; \emph{keyed MSE} adds
$(w_P/k)\|P^\top e_\theta\|_2^2$ with $w_P=10$. All rows use
$\lambda_{\rm corr}=0$. Values are means and population standard
deviations across three keys.}
\label{tab:cifar-objective-ablation}
\setlength{\tabcolsep}{5pt}
\begin{tabular}{lrrr}\toprule
Objective & $\widehat\epsilon$ & Message/bits (\%) & NMSE (\%) \\
\midrule
Flow-matching target, velocity MSE & $0.0412\pm0.0033$ & $0.0/61.56\pm1.22$ & $1.433\pm0.005$ \\
Flow-matching target, velocity + keyed MSE & $1.4192\pm0.0042$ & $100.0/100.0$ & $2.113\pm0.077$ \\
Clean reference, velocity + keyed MSE & $1.4665\pm0.0022$ & $100.0/100.0$ & $0.151\pm0.013$ \\
\bottomrule\end{tabular}
\end{table}

The ordinary velocity MSE averages error over all 3,072 output coordinates,
so error in the 32 keyed directions receives little weight. The keyed-subspace
term measures that error separately, yielding the prescribed component and
complete recovery. Using the clean model's velocity
as the regression target further reduces endpoint error by removing the
sample-pair variation in $x_1-x_0$, so Equation~\eqref{eq:loss} uses this
formulation. The primary full-weight CIFAR-10 checkpoints use
Equation~\eqref{eq:loss} with $w_P=10$ and $\lambda_{\rm corr}=0$. Across five H32 messages, they achieve complete
recovery and endpoint NMSE $0.162\pm0.018\%$ before fixed mean correction.
Direct measurement on 4,096 held-out interpolation queries gives
$\widehat\epsilon=1.4656\pm0.0020$ for
$a_{\mathrm{nom}}=1.5$, carrier RMS equal to
$1.8112\pm0.0025\%$ of clean-field RMS, and total field-change RMS equal
to $6.55\pm0.18\%$ of clean-field RMS.

\subsection{Controlled Temporal-Carrier Learning on CIFAR-10}
\label{app:multiplexing}

This experiment isolates temporal-carrier learnability in the pretrained
CIFAR-10 RF-UNet. Let $P\in\mathbb R^{D\times4}$ be a common orthonormal
projection and $s\in\{-1,1\}^4$ the payload. Every condition fits
\begin{equation}
v_{\mathrm{target}}(x,t)=u(x,t)+\frac{A_0}{2}
P\bigl(s\odot\boldsymbol\alpha(t)\bigr),\qquad A_0=1.5,
\end{equation}
where $u$ is the frozen clean model. For coordinate $b\in\{0,1,2,3\}$,
the single carrier uses $\alpha_b(t)=\sqrt2\sin(2\pi t)$, FDM-4 uses
$\alpha_b(t)=\sqrt2\sin(2\pi(b+1)t)$, and TDM-4 uses
$\alpha_b(t)=\sqrt8\sin(2\pi(4t-b))$ on $t\in[b/4,(b+1)/4)$ and zero
otherwise. Thus all conditions use four spatial directions and satisfy
$\mathbb E_t\|A_0(s\odot\boldsymbol\alpha(t))/2\|_2^2=A_0^2$.

All models update the full network for 1000 AdamW steps with batch 64,
initial learning rate $10^{-4}$, weight decay 0.01, and cosine decay.
The loss is velocity MSE to $v_{\mathrm{target}}$ plus a
keyed-subspace target-error term of weight 10.
Three keys share the clean initialization and training queries across
conditions. On 4096 held-out interpolation queries, we measure
$100\,\mathbb E\|P^\top(v_\theta-u)-A_0(s\odot\boldsymbol\alpha(t))/2\|_2^2/
\mathbb E\|A_0(s\odot\boldsymbol\alpha(t))/2\|_2^2$. Paired endpoint NMSE uses 128
common initial noises and 100 Euler steps.

\begin{table}[!htb]\centering\small
\caption{Controlled temporal-carrier learning on CIFAR-10. Keyed-coefficient and endpoint
NMSE are percentages; means and population standard deviations are across
three keys. Message recovery is 100\% in every cell.}
\label{tab:temporal-multiplexing}
\setlength{\tabcolsep}{5pt}
\begin{tabular}{@{}lrrrr@{}}\toprule
& \multicolumn{3}{c}{Keyed-coefficient NMSE after updates} & Endpoint \\
\cmidrule(lr){2-4}
Carrier & 100 & 500 & 1000 & NMSE \\
\midrule
Single sinusoid & $4.669\pm0.158$ & $2.004\pm0.020$ & $1.096\pm0.030$ & $1.258\pm0.040$ \\
FDM-4 & $21.965\pm1.652$ & $9.885\pm2.276$ & $4.690\pm0.059$ & $2.853\pm0.343$ \\
TDM-4 & $60.751\pm3.251$ & $15.777\pm0.448$ & $8.153\pm0.152$ & $3.988\pm0.654$ \\
\bottomrule\end{tabular}
\end{table}

The single sinusoid has the lowest held-out target error throughout training.
At 1000 updates its fitted gain is $0.992\pm0.001$, compared with
$0.967\pm0.002$ for FDM-4 and $0.941\pm0.003$ for TDM-4. The more complex
temporal targets remain decodable but are approximated less accurately and
produce larger endpoint changes.

\section{Modification, Removal, and Transfer}\label{app:full-weight-studies}

This section collects the protocols and supplementary results for
Sections~\ref{sec:modification} and~\ref{sec:comparative-transfer}.
In the primary CIFAR-10 studies, watermark embedding updates every source-model
parameter; subsequent fine-tuning and student training also update all
parameters. We first report this full-weight pipeline and carrier subtraction, then the
four-method comparison, student transfer, and channel permutations.
The final subsection groups supplementary LoRA-embedded models by
recipient modification scope: the entire model or only adapters and
selected layers.

\subsection{Full-Weight Embedding and Whole-Model Modification}\label{app:whole-model-modification}

Table~\ref{tab:whole-model-modification} evaluates whole-model modifications
after full-weight watermark embedding into the pretrained CIFAR-10 backbone. We evaluate five O5 and five H32
messages. Embedding uses 1,000 AdamW steps with batch 64, initial learning
rate $10^{-4}$, weight decay 0.01, and cosine decay. The nominal signal
strength is $a_{\mathrm{nom}}=1.5$. Embedding uses
Equation~\eqref{eq:loss} with $w_P=10$ and $\lambda_{\rm corr}=0$.
Each modification independently starts from
its corresponding watermarked checkpoint. Pruning and simulated rounding affect every
floating-point parameter, including biases and normalization parameters.
Clean-data fine-tuning updates all parameters for 1,000 ordinary
flow-matching steps with AdamW and cosine learning-rate decay.

\begin{table}[!htb]\centering\small
\caption{Whole-model modification after full-weight embedding. Mean and
population standard deviation across five messages per code. FID ratios
use the unmodified clean model as the reference.}
\label{tab:whole-model-modification}
\setlength{\tabcolsep}{4pt}
\begin{tabular}{llrrr}\toprule
Code & Condition & Message (\%) & Bits (\%) & FID/clean \\
\midrule
O5 & Original & $100.0\pm0.0$ & $100.00\pm0.00$ & $1.008\pm0.003$ \\
O5 & 25\% pruning & $100.0\pm0.0$ & $100.00\pm0.00$ & $1.012\pm0.006$ \\
O5 & 50\% pruning & $10.0\pm12.2$ & $57.60\pm8.96$ & $1.833\pm0.016$ \\
O5 & 8-bit rounding & $100.0\pm0.0$ & $100.00\pm0.00$ & $1.009\pm0.005$ \\
O5 & 4-bit rounding & $100.0\pm0.0$ & $100.00\pm0.00$ & $1.159\pm0.018$ \\
O5 & 2-bit rounding & $0.0\pm0.0$ & $48.40\pm4.08$ & $3.744\pm0.021$ \\
O5 & Fine-tuning, $10^{-5}$ & $100.0\pm0.0$ & $100.00\pm0.00$ & $0.985\pm0.001$ \\
O5 & Fine-tuning, $10^{-4}$ & $77.0\pm15.0$ & $88.40\pm8.80$ & $0.968\pm0.002$ \\
\midrule
H32 & Original & $100.0\pm0.0$ & $100.00\pm0.00$ & $1.006\pm0.004$ \\
H32 & 25\% pruning & $100.0\pm0.0$ & $100.00\pm0.00$ & $1.006\pm0.005$ \\
H32 & 50\% pruning & $0.0\pm0.0$ & $62.19\pm9.34$ & $1.862\pm0.012$ \\
H32 & 8-bit rounding & $100.0\pm0.0$ & $100.00\pm0.00$ & $1.009\pm0.004$ \\
H32 & 4-bit rounding & $100.0\pm0.0$ & $100.00\pm0.00$ & $1.157\pm0.019$ \\
H32 & 2-bit rounding & $0.0\pm0.0$ & $50.66\pm8.13$ & $3.777\pm0.037$ \\
H32 & Fine-tuning, $10^{-5}$ & $12.0\pm8.1$ & $94.88\pm0.72$ & $0.987\pm0.001$ \\
H32 & Fine-tuning, $10^{-4}$ & $0.0\pm0.0$ & $71.25\pm1.55$ & $0.967\pm0.001$ \\
\bottomrule\end{tabular}
\end{table}

\subsubsection{Robustness of Matched LoRA and Full-Weight Embedding}
\label{app:matched-embedding}

We repeat embedding with LoRA and selected time-embedding and
output-projection layers, starting from the same pretrained CIFAR-10
model and using the same five messages per code, keys, training queries,
objective, nominal strength, and optimization budgets as the full-weight
experiment. Before whole-model pruning, quantization, and clean-data
fine-tuning, each learned LoRA update is folded into its corresponding base
weight, producing an ordinary model without separate adapter modules.
Consequently the recipient's
modification scope is identical for both embedding strategies.
Table~\ref{tab:embedding-scope-comparison} compares message retention under
these matched modifications. Before modification, the FID/clean ratios are
$1.008$ and $1.006$ for the full-weight O5 and H32 models, and $1.025$ and
$1.030$ for their matched LoRA counterparts. On 256 paired inputs per model
with 100 Euler steps, endpoint NMSE is $0.1667\pm0.0206\%$ and
$0.1704\pm0.0178\%$ for the full-weight O5 and H32 models, compared with
$1.4461\pm0.2481\%$ and $1.4384\pm0.3953\%$ for their LoRA counterparts.
Means and population standard deviations are across the same five messages.
The LoRA checkpoints in this comparison are retrained with the matched
objective.

\begin{table}[!htb]\centering\small
\caption{Message retention for matched embedding strategies under whole-model modification.
Cells give mean message recovery / bit accuracy (\%) across five messages
per code. Full and LoRA identify the embedding strategy; fine-tuning
updates all model weights for 1000 steps in both cases.}
\label{tab:embedding-scope-comparison}
\begin{tabular}{@{}lrrrr@{}}\toprule
Condition & O5 full & O5 LoRA & H32 full & H32 LoRA \\
\midrule
Original & 100/100.0 & 100/100.0 & 100/100.0 & 100/100.0 \\
25\% pruning & 100/100.0 & 100/100.0 & 100/100.0 & 100/100.0 \\
4-bit rounding & 100/100.0 & 100/100.0 & 100/100.0 & 100/100.0 \\
Fine-tuning, $10^{-5}$ & 100/100.0 & 55/77.6 & 12/94.9 & 0/66.4 \\
Fine-tuning, $10^{-4}$ & 77/88.4 & 6/50.6 & 0/71.3 & 0/51.9 \\
Sine subtraction & 2/49.4 & 2/49.4 & 0/48.0 & 0/47.4 \\
Cosine subtraction & 100/100.0 & 100/100.0 & 100/100.0 & 100/100.0 \\
\bottomrule\end{tabular}
\end{table}

Both strategies retain complete recovery under moderate compression and
lose recovery under matched carrier subtraction. Full-weight embedding
retains more message information under the tested clean-data fine-tuning
budgets and produces lower paired endpoint error under the matched embedding
budget. Appendix~\ref{app:restricted-modifications} evaluates modifications
restricted to adapter and time-embedding parameters.

\subsection{Carrier Subtraction after Full-Weight Embedding}\label{app:full-weight-removal}

For the full-weight-embedded source models, we estimate the signal direction
using 16,384 model interpolation queries and apply
Equation~\eqref{eq:carrier-removal} with $\beta=1$. The matched attack uses
sine; the mismatched control uses cosine. Both regressions use suspect-model
responses and the public candidate carrier. Table~\ref{tab:full-weight-removal}
reports recovery on independent verification queries and the corresponding
FID ratios.

\begin{table}[!htb]\centering\small
\caption{Signal subtraction after full-weight embedding. Mean and
population standard deviation across five messages. FID ratios use the
original clean model.}
\label{tab:full-weight-removal}
\setlength{\tabcolsep}{4pt}
\begin{tabular}{llrrr}\toprule
Code & Subtracted carrier & Message (\%) & Bits (\%) & FID/clean \\
\midrule
O5 & Sine & $2.0\pm2.4$ & $49.40\pm4.72$ & $1.007\pm0.002$ \\
O5 & Cosine & $100.0\pm0.0$ & $100.00\pm0.00$ & $1.008\pm0.003$ \\
\midrule
H32 & Sine & $0.0\pm0.0$ & $48.00\pm2.06$ & $1.007\pm0.001$ \\
H32 & Cosine & $100.0\pm0.0$ & $100.00\pm0.00$ & $1.006\pm0.004$ \\
\bottomrule\end{tabular}
\end{table}

Velocity-supervised transfer from full-weight source models is evaluated in
Table~\ref{tab:comparative-transfer}. A separate transfer study from LoRA-embedded
source models, also with full-weight student training, is retained in
Appendix~\ref{app:lora-teacher-transfer}.

\subsection{Four-Method Comparison}\label{app:comparison}

Table~\ref{tab:comparison} summarizes the signal locations and verification
access of the methods considered here.

\begin{table}[!htb]\centering\small
\caption{Signal locations and verification access.}
\label{tab:comparison}
\begin{tabular}{@{}lll@{}}\toprule
Method & Signal location & Verification access \\
\midrule
Tree-Ring~\citep{wen2023tree} & Sampling noise & Images and inversion \\
Stable Signature~\citep{fernandez2023stable} & Image decoder & Images and extractor \\
Weight embedding~\citep{uchida2017embedding} & Parameter statistics & Model weights \\
Trigger-based~\citep{adi2018turning,chen2019blackmarks} & Designated responses & Trigger queries \\
SleeperMark~\citep{wang2025sleepermark} & Triggered latent residual & Triggered images \\
Velocity field (ours) & Time-modulated velocity & Velocity queries \\
\bottomrule\end{tabular}
\end{table}

\subsubsection{Embedding and Student Protocols}
\label{app:comparative-study}

Table~\ref{tab:comparative-transfer} compares the velocity watermark with
three alternatives under the same types of modification and student training.
The weight method adds a binary cross-entropy penalty on projected
convolutional weights to the flow-matching loss, adapting
\citet{uchida2017embedding}. The trigger method trains
trajectories from a fixed noise neighborhood to a spatial bit pattern,
adapting \citet{adi2018turning}. The third method adds a
time-independent signal to the velocity target, using public spatial
directions and correlation training. It adapts spread-spectrum
signaling~\citep{cox1997spread} to velocities.

Reported quality ratios use 5,000 images. Embedding and student training
update all parameters.
For weight, trigger, and time-independent baselines, fine-tuning and student
training use three seeds that vary training draws and dropout; reported
recovery is averaged across these seeds. Compression uses the fixed source
models directly.
The students in Table~\ref{tab:comparative-transfer} start from the public
clean checkpoint and share the source architecture. They train for 1,000 steps using either source
velocities or 2,000 generated images.
Image datasets use the original source outputs. Their results are in
Appendix~\ref{app:uncorrected-transfer}. Verification uses 20 query trials;
weight decoding reads the saved parameters directly.

\paragraph{Image watermarking and the time-independent signal.}
Classical spread-spectrum watermarking embeds a keyed pattern in image
coefficients and detects it by correlation; the construction of
\citet{cox1997spread} uses an original-image reference. Our
\emph{time-independent velocity signal} spreads a code over fixed
spatial directions with a time-independent amplitude. Nonlinear transport maps
its perturbation to generated outputs. It is therefore a velocity-target
adaptation of spread-spectrum signaling.

\subsection{Velocity Distillation from Random Initialization}
\label{app:random-initialization}

Same-architecture students use three paired random initializations and fit
clean or watermarked source velocities for 1,000 full-weight AdamW updates.
They use the 64-channel CIFAR-10 UNet with two residual blocks per level;
all parameters are trainable. Training uses noise--image interpolations,
batch size 64, initial learning rate $10^{-4}$, weight decay 0.01, and
cosine decay. Within each initialization, clean- and watermarked-source
students share the training queries and dropout seed.
Verification uses 20 trials of 4,096 held-out interpolation queries;
quality for the reference initialization uses 1,000 balanced CIFAR-10 images
with paired generation noise and class labels. Relative MSE is evaluated on
4,096 velocity queries for every initialization.
Students for the weight-based baseline use the same architecture and random
initialization protocol and fit velocities from the weight-watermarked source
or its clean control.
Table~\ref{tab:random-initialization} reports both methods.

\begin{table}[!htb]\centering\small
\caption{Same-architecture students trained from random weights. Recovery
cells give message recovery / bit accuracy (\%). Recovery and relative MSE
are means and population standard deviations across three paired
initializations; FID/source uses the 1,000-image reference evaluation.}
\label{tab:random-initialization}
\setlength{\tabcolsep}{4pt}
\begin{tabular}{@{}llrrr@{}}\toprule
Method & Training source & Recovery (\%) & FID/source & Relative MSE \\
\midrule
Velocity field & Clean & $0/(47.86\pm0.97)$ & 1.415 & $0.02058\pm0.00012$ \\
Velocity field & Watermarked & $100/(100.00\pm0.00)$ & 1.408 & $0.02090\pm0.00019$ \\
Weight & Clean & $0/(43.75\pm5.10)$ & 1.347 & $0.02195\pm0.00006$ \\
Weight & Watermarked & $0/(44.79\pm3.90)$ & 1.344 & $0.02204\pm0.00005$ \\
\bottomrule\end{tabular}
\end{table}

The velocity message therefore transfers across the tested random
initializations without a shared pretrained checkpoint, whereas the weight
message does not. FID/source is student FID divided by source FID.

\subsection{Sample-Only Transfer}\label{app:sample-transfer}
\subsubsection{Training Students on Corrected Outputs}
\label{app:corrected-transfer}

The three-key transfer study uses three source models embedded by updating
all parameters, with independently sampled projection/message keys. Each starts from the
clean checkpoint and embeds the additive-sine signal by regressing to
$u(x,t)+1.5\sin(2\pi t)Pc_m$. The embedding loss is mean per-coordinate
velocity squared error plus 10 times mean per-coordinate keyed-subspace squared
error. Embedding uses 1,000 updates, learning rate $10^{-4}$, and batch 64,
with training draws fixed across keys.
All three source models recover the complete message on every verification trial.

For each key, we compare sample-only students trained on clean, uncorrected
watermarked, or corrected watermarked generated images. Each dataset contains
2,000 images generated from matching noise and class labels. The fixed RGB
offset is estimated on 256 separate paired outputs from the frozen source
model and subtracted to form the corrected dataset. Students start from the
same public clean checkpoint and update all weights for 1,000 steps with
AdamW, initial learning rate $10^{-4}$, weight decay 0.01, batch 64, and
cosine decay. Three student training seeds are paired across conditions.
The same three students trained on clean outputs are evaluated under
each key. A velocity-supervised student is also trained for every key and
seed; all nine recover the complete 32-bit message.

Student utility uses a 1,000-image CIFAR-10 test reference with equal
numbers from each class and generation noise paired across conditions. FID
is normalized by each student's paired clean-data student.

\begin{table}[!htb]\centering\small
\caption{Sample-only transfer before and after mean output correction.
Means across three independent keys, each averaging three paired training
seeds. Excess is the paired bit-accuracy increase over clean-data students
in percentage points.}
\label{tab:sample-transfer-summary}
\begin{tabular}{lrrr}\toprule
Student data & Bits (\%) & Excess (pp) & FID/clean student \\
\midrule
Clean outputs & 47.53 & 0 & 1.0000 \\
Uncorrected outputs & 73.66 & 26.13 & 1.0031 \\
Corrected outputs & 50.38 & 2.85 & 1.0107 \\
\bottomrule\end{tabular}
\end{table}

Table~\ref{tab:sample-transfer-summary} shows that mean correction reduces
excess bit agreement over clean-data students from 26.13 to 2.85 percentage
points, an 89\% reduction, with similar FID ratios.

\begin{table}[!htb]\centering\small
\caption{Sample-only velocity-field watermark recovery before and after
mean output correction. Entries average three paired student seeds within each
independent key. Clean, raw, and corrected columns are bit accuracy (\%);
the last two columns are paired differences in percentage points.}
\label{tab:corrected-sample-transfer}
\setlength{\tabcolsep}{4pt}
\begin{tabular}{lrrrrr}\toprule
Key & Clean & Raw & Corrected & Raw$-$clean & Corrected$-$clean \\
\midrule
1 & 52.08 & 75.78 & 56.46 & 23.70 & 4.38 \\
2 & 44.48 & 71.20 & 46.30 & 26.72 & 1.82 \\
3 & 46.04 & 74.01 & 48.39 & 27.97 & 2.34 \\
\midrule
Mean & 47.53 & 73.66 & 50.38 & 26.13 & 2.85 \\
\bottomrule\end{tabular}
\end{table}

Table~\ref{tab:corrected-sample-transfer} gives the per-key measurements
underlying Table~\ref{tab:sample-transfer-summary}. Across keys, the standard
deviations of the bit-accuracy increases over clean-data students are 1.79
percentage points before correction and 1.10 afterward.

\paragraph{Weight-Watermark Transfer to Students.}
The weight experiment uses a source whose watermark was embedded by updating
all parameters, together with a matched
clean source. Students start from the public clean checkpoint and train for
1,000 steps under velocity, uncorrected-sample, or corrected-sample
supervision. Table~\ref{tab:weight-transfer-replication} reports recovery
using the source model's watermark key.

\begin{table}[!htb]\centering\small
\caption{Recovery of the source model's weight watermark in students
initialized from the public clean checkpoint. Each supervision condition
uses three paired training seeds; standard deviations are across seeds.}
\label{tab:weight-transfer-replication}
\begin{tabular}{lrr}\toprule
Student condition & Bit accuracy (\%) & Message recovery (\%) \\
\midrule
Velocities from clean source & $53.125\pm0.000$ & 0 \\
Velocities from watermarked source & $100.000\pm0.000$ & 100 \\
Images from clean source & $53.125\pm0.000$ & 0 \\
Uncorrected watermarked samples & $61.458\pm1.804$ & 0 \\
Corrected watermarked samples & $61.458\pm1.804$ & 0 \\
\bottomrule\end{tabular}
\end{table}

Table~\ref{tab:weight-transfer-replication} reproduces 100\%
velocity-supervised weight recovery across training
seeds under common initialization, while both sample-only conditions retain
less bit agreement. Correction leaves the decoded weight bits unchanged in these runs.
For the weight source model, image NMSE is
0.00829\% before correction and 0.00831\% after correction. Table~\ref{tab:random-initialization} reports
the completed same-architecture experiment initialized from random weights.
Table~\ref{tab:interface-control} tests dependence on parameter coordinates.

\subsubsection{Additional Uncorrected-Output Studies}
\label{app:uncorrected-transfer}
These studies use the LoRA-embedded CIFAR-10 checkpoints from
Tables~\ref{tab:main} and~\ref{tab:all-generators}, with students trained
on uncorrected outputs.

\paragraph{Sample-Only Transfer from LoRA-Embedded Sources.}
This study uses the LoRA-embedded O5 and H32 source checkpoints.
For each message, the source generates 2,000 images. Students start from
the public clean checkpoint and update all weights using those images and
class labels.

\begin{table}[!htb]\centering\small
\caption{Sample-only transfer from \emph{uncorrected} source model
outputs. Mean and standard deviation across five messages.
FID is normalized by the original clean model.}
\label{tab:uncorrected-multimessage-transfer}
\begin{tabular}{lrrr}\toprule
Code & Message recovery (\%) & Bit accuracy (\%) & FID/clean \\
\midrule
O5 & $100.0\pm0.0$ & $100.00\pm0.00$ & $1.075\pm0.007$ \\
H32 & $4.0\pm2.0$ & $92.09\pm1.22$ & $1.080\pm0.016$ \\
\bottomrule\end{tabular}
\end{table}

Table~\ref{tab:uncorrected-multimessage-transfer} shows complete O5
transfer; H32 retains high bit accuracy
but rarely recovers the entire 32-bit message.

\paragraph{Four-Method Sample-Only Transfer.}
The comparative study also trained on 2,000 uncorrected source model
outputs. Table~\ref{tab:uncorrected-comparative-transfer} reports the same-architecture
sample-only recovery and FID ratios.

\begin{table}[!htb]\centering\small
\caption{Comparative sample-only transfer from uncorrected outputs. Recovery
averages three training seeds. FID uses the 5,000-image evaluation and is
normalized by the matched clean-data student for every method.}
\label{tab:uncorrected-comparative-transfer}
\begin{tabular}{lrrr}\toprule
Method & Message recovery (\%) & Bit accuracy (\%) & FID ratio \\
\midrule
Velocity field & 0 & 75.8 & 0.998 \\
Weight & 0 & 61.5 & 1.003 \\
Trigger & 0 & 56.1 & 1.011 \\
\shortstack{Time-independent signal} & 98.3 & 99.9 & 1.031 \\
\bottomrule\end{tabular}
\end{table}

\subsection{Function-Preserving Channel Permutations}
\label{app:interface-control}

The weight verifier uses a secret projection of a selected convolution's
weights. The velocity-supervised student starts from the public clean
checkpoint. We apply 20 fixed internal-channel permutations to the saved full-weight-embedded
source model and its student. Each
permutation maps complete GroupNorm groups to groups and permutes channels
within them. The input convolution and time embedding use the corresponding
output-channel permutation; the output convolution uses its inverse on
input channels. The residual block's input and output coordinates remain
fixed. Pointwise activations and GroupNorm commute with this construction,
so it preserves the represented velocity function.

Table~\ref{tab:interface-control} reports decoding under the fixed key on
4096 CIFAR-10 interpolation queries. The maximum velocity and demodulated-score
changes across the source and student are $1.43\times10^{-6}$ and
$4.82\times10^{-9}$, respectively.
After the same permutations, bit accuracy for the weight-based baseline
drops to 55.000\% for the source and 53.594\% for the student.

\begin{table}[!htb]\centering\small
\caption{Fixed-key verification under function-preserving hidden-channel
permutations. Bit accuracy is averaged over 20 permutations; the last
column counts complete messages.}
\label{tab:interface-control}
\begin{tabular}{llrrr}\toprule
Verifier & Model & Original bits (\%) & Permuted bits (\%) & Recovered / 20 \\
\midrule
Weight & Source & 100.0 & 55.000 & 0 \\
Weight & Student & 100.0 & 53.594 & 0 \\
Velocity & Source & 100.0 & 100.000 & 20 \\
Velocity & Student & 100.0 & 100.000 & 20 \\
\bottomrule\end{tabular}
\end{table}

\subsection{Modification and Transfer of LoRA-Embedded Models}\label{app:lora-source-studies}

These supplementary studies evaluate how the recipient's modification
scope affects LoRA-embedded CIFAR-10 UNet watermarks. Embedding optimizes
LoRA matrices and selected time-embedding, attention-output, and final
output layers. We consider whole-model modification for O5 and H32, and
modification restricted to this parameter subset for O5. Each condition
starts independently from its watermarked source checkpoint. The matched
embedding-strategy comparison is reported in
Appendix~\ref{app:matched-embedding}.

\subsubsection{Whole-Model Modification and Student Transfer}
\label{app:lora-teacher-transfer}

Table~\ref{tab:removal-transfer} evaluates the LoRA-embedded CIFAR-10
checkpoints reported in Tables~\ref{tab:main} and~\ref{tab:all-generators}.

For clean-data fine-tuning, we first fold the LoRA updates into the base
weights and then update all parameters
for 1000 flow-matching steps. Velocity-supervised students use the source
architecture, start from the public clean checkpoint, and update all
parameters for 1000 steps to fit source responses on real-data
interpolations. Both use AdamW with initial learning rate $10^{-4}$ and
cosine decay. Carrier subtraction follows
Appendix~\ref{app:full-weight-removal}, including the sine regression and
cosine control. Quality evaluation uses class-balanced samples and paired
generation noise.

\begin{table}[!htb]
\centering
\caption{Targeted removal and velocity-supervised transfer to students that
update all parameters, using
the LoRA-embedded CIFAR-10 checkpoints from Tables~\ref{tab:main}
and~\ref{tab:all-generators}. FID is
normalized by the unmodified watermarked source model.}
\label{tab:removal-transfer}
\small
\setlength{\tabcolsep}{4pt}
\begin{tabular}{@{}llccc@{}}
\toprule
Code & Condition & \shortstack{Message recovery\\(\%)} & \shortstack{Bit accuracy\\(\%)} & FID ratio \\
\midrule
O5 & Watermarked source model & $100.0{\pm}0.0$ & $100.00{\pm}0.00$ & $1.000{\pm}0.000$ \\
& Sine-signal subtraction & $2.0{\pm}2.4$ & $49.80{\pm}6.01$ & $0.997{\pm}0.008$ \\
& Cosine-signal subtraction & $100.0{\pm}0.0$ & $100.00{\pm}0.00$ & $1.000{\pm}0.000$ \\
& Clean-data fine-tuning & $6.0{\pm}4.9$ & $50.80{\pm}6.79$ & $0.939{\pm}0.008$ \\
& Velocity-supervised student & $100.0{\pm}0.0$ & $100.00{\pm}0.00$ & $1.000{\pm}0.001$ \\
\midrule
H32 & Watermarked source model & $100.0{\pm}0.0$ & $100.00{\pm}0.00$ & $1.000{\pm}0.000$ \\
& Sine-signal subtraction & $0.0{\pm}0.0$ & $47.25{\pm}2.00$ & $0.999{\pm}0.006$ \\
& Cosine-signal subtraction & $100.0{\pm}0.0$ & $100.00{\pm}0.00$ & $1.000{\pm}0.000$ \\
& Clean-data fine-tuning & $0.0{\pm}0.0$ & $51.59{\pm}1.88$ & $0.935{\pm}0.010$ \\
& Velocity-supervised student & $100.0{\pm}0.0$ & $100.00{\pm}0.00$ & $1.000{\pm}0.001$ \\
\bottomrule
\end{tabular}
\end{table}

Table~\ref{tab:removal-transfer} shows near-chance bit accuracy after
sine-carrier subtraction or clean-data fine-tuning of all parameters, while
velocity-supervised students retain both codes completely. Cosine
subtraction and half-strength sine subtraction preserve complete recovery.
Sample-only transfer from these LoRA-embedded sources is reported in
Appendix~\ref{app:uncorrected-transfer}.

\subsubsection{Modification Restricted to Adapters and Selected Layers}
\label{app:restricted-modifications}

For the five O5 sources in this study, fine-tuning, pruning, and
quantization affect only the embedding parameter subset defined above;
the remaining backbone is frozen. Tables~\ref{tab:clean-finetune}
and~\ref{tab:pruning-quantization} report mean and population standard
deviation across messages. Here FID is normalized by the original clean
model, with the CIFAR-10 test split as the reference distribution.

\paragraph{Clean-data fine-tuning.}
Table~\ref{tab:clean-finetune} varies the initial learning rate and update
budget for clean flow-matching adaptation with cosine learning-rate decay.

\begin{table}[!htb]
\centering
\caption{Restricted-parameter clean fine-tuning of O5 watermarks.}
\label{tab:clean-finetune}
\small
\begin{tabular}{cccc}
\toprule
Initial LR & Updates & Message recovery & FID ratio \\
\midrule
--- & 0 & $100.0{\pm}0.0$\% & $1.012{\pm}0.012$ \\
$10^{-5}$ & 1,000 & $100.0{\pm}0.0$\% & $0.977{\pm}0.007$ \\
$10^{-5}$ & 5,000 & $100.0{\pm}0.0$\% & $0.954{\pm}0.004$ \\
\midrule
$10^{-4}$ & 100 & $100.0{\pm}0.0$\% & $0.986{\pm}0.007$ \\
$10^{-4}$ & 250 & $100.0{\pm}0.0$\% & $0.953{\pm}0.005$ \\
$10^{-4}$ & 500 & $100.0{\pm}0.0$\% & $0.955{\pm}0.005$ \\
$10^{-4}$ & 1,000 & $96.0{\pm}4.9$\% & $0.956{\pm}0.001$ \\
\bottomrule
\end{tabular}
\end{table}

At initial learning rate $10^{-5}$, message recovery remains 100\% through
5000 updates. At $10^{-4}$, recovery remains 100\% through 500 updates and
falls to $96.0\pm4.9\%$ at 1000 updates. FID ratios decrease below one
under both schedules, indicating improved FID relative to the clean model.

\paragraph{Pruning and quantization.}
We apply global magnitude pruning within the same parameter set.
Quantization uses signed symmetric rounding followed by dequantization
within each saved tensor. At 2 bits, the reconstruction levels are
$\{-a,0,a\}$, where $a$ is the tensor's maximum absolute weight.
Table~\ref{tab:pruning-quantization} reports recovery and generation quality
under these compression operations.

\begin{table}[!htb]
\centering
\caption{Compression of the embedding-trained parameters of O5
watermarks. Pruning fractions refer to this parameter set.}
\label{tab:pruning-quantization}
\small
\begin{tabular}{lcc}
\toprule
Modification & Message recovery & FID ratio \\
\midrule
None & $100.0{\pm}0.0$\% & $1.012{\pm}0.012$ \\
Magnitude pruning, 25\% & $100.0{\pm}0.0$\% & $1.018{\pm}0.009$ \\
Magnitude pruning, 50\% & $77.0{\pm}28.6$\% & $1.344{\pm}0.049$ \\
Simulated quantization, 8-bit & $100.0{\pm}0.0$\% & $1.013{\pm}0.012$ \\
Simulated quantization, 4-bit & $100.0{\pm}0.0$\% & $1.073{\pm}0.022$ \\
Simulated quantization, 2-bit & $68.0{\pm}41.2$\% & $2.854{\pm}0.134$ \\
\bottomrule
\end{tabular}
\end{table}

Message recovery remains 100\% after 25\% pruning and 8-bit or 4-bit
quantization. Increasing pruning to 50\% reduces recovery to 77\% and
raises the FID ratio to 1.344; 2-bit quantization reduces recovery to 68\%
and raises the ratio to 2.854. Thus moderate modifications of the embedding
parameters retain the mark, while stronger compression reduces both
recovery and generation fidelity.

The distinction between these protocols is most evident in fine-tuning:
restricted O5 adaptation retains 96\% message recovery after 1000 updates at
$10^{-4}$, whereas the whole-model experiment in
Table~\ref{tab:removal-transfer} retains 6\%. The matched LoRA/full-weight comparison in
Table~\ref{tab:embedding-scope-comparison} holds whole-model modification
fixed and evaluates the separate effect of embedding parameterization.

\section{Notation}\label{app:notation}
{\small
\begin{longtable}{@{}p{0.29\textwidth}p{0.67\textwidth}@{}}\toprule
Symbol & Meaning \\
\midrule
\endfirsthead
\toprule
Symbol & Meaning \\
\midrule
\endhead
$x,t$ & State and flow time \\
$D,k,L$ & State dimension, code-space dimension, message length in bits \\
$x_0,x_1,x_t,U$ & Base sample, data sample, interpolation, conditional target velocity \\
$p_{\mathrm{noise}},q_{\mathrm{data}},p_q$ & Base, data, and verification-query distributions \\
$K,m,\widehat m,\mathcal M$ & Secret key, message, decoded message, and message set \\
$P,C,c_m,q$ & Keyed projection, codebook, codeword, and direction $q=Pc_m$ \\
$s_{K,m},\phi_K,\mathcal H$ & Watermark signal, keyed encoding map, and signal function space \\
$E_s,\Delta_K,\delta_{\mathrm{dist}}$ & Signal energy, minimum separation, and TV distortion budget \\
$u,v_0,v_w$ & Clean reference velocities and fitted watermarked velocity \\
$v_\theta,v^\star$ & Parameterized model velocity and population minimizer \\
$a_{\mathrm{nom}},w_P,\lambda_{\rm corr}$ & Nominal target coefficient, keyed-subspace weight, and correlation-reward weight \\
$\Gamma(\theta)$ & Demodulated correlation of the fitted velocity change with the encoded message \\
$\epsilon,\widehat\epsilon$ & Carrier amplitude and its measured estimate \\
$\eta,e_\theta$ & Training and velocity-fitting residual fields \\
$\alpha(t),\boldsymbol\alpha(t),\beta$ & Scalar carrier, carrier vector, and carrier-subtraction strength \\
$N,J,h$ & Verification queries, Euler steps, and step size ($Jh=1$) \\
$\widehat s,s_{\mathrm{dec}}$ & Averaged demodulated score and score after subtracting $\widehat b_K$ \\
$b_K,\widehat b_K,\xi$ & Mean background score, its estimate, and centered score noise \\
$\kappa_\alpha,\widehat\kappa_\alpha,\Sigma_\alpha$ & Carrier energy, empirical carrier energy, and background covariance \\
$\sigma,\Phi$ & Isotropic background-noise scale and standard-normal CDF \\
$Q,R_o,\mathcal C_\alpha$ & Codeword covariance, orthogonal key rotation, and Gaussian-channel capacity \\
$\delta u$ & Vector-valued velocity perturbation $v-u$ \\
$\rho,\rho^w,\rho^\epsilon$ & Clean and perturbed densities \\
$C_{K,m}$ & Density-weighted carrier-forcing coefficient \\
$\chi,\mathcal R^\epsilon,C_1,C_2$ & Scalar first-order density response, remainder, response norm, and remainder coefficient \\
$\varphi_t,\varphi_t^w,\varphi_{t\leftarrow r}$ & Clean and perturbed flow maps; clean flow from time $r$ to $t$ \\
$\mathcal T_{t\leftarrow r},\mathcal T^w_{t\leftarrow r},G_q(r)$ & Clean and perturbed signed-density transport and directional forcing \\
$\ell(t),\ell_{\max},L_q$ & Velocity Lipschitz bounds and transported-forcing Lipschitz constant \\
$X_t,Y_t,Z$ & Coupled clean and perturbed trajectories and their shared initial state \\
$r(x,t),A(t)$ & Velocity difference and Gr\"onwall amplification factor \\
$\delta(t),v_{\mathrm{ep}}$ & Translation path and endpoint-preserving reference field \\
$\eta_u(h),\eta_v(h)$ & Numerical endpoint $L^2$ errors for clean and perturbed flows \\
$R,\widehat R,\widehat R_{\mathrm{mean}}$ & Endpoint RMS, its empirical estimate, and RMS of mean paired error \\
$n_0,S_0,R_{\mathrm{rel}},E_{\mathrm{rel}}$ & Clean reference count/RMS, relative endpoint RMS, and normalized error energy \\
$R_{\mathrm{ep}}(f)$ & Endpoint response energy normalized by prescribed signal energy \\
$b$ & Fixed RGB mean offset for output correction \\
$r_{\mathrm{carr}},r_{\mathrm{change}}$ & Carrier and total-change RMS relative to clean velocity RMS \\
$A_0$ & Prescribed RMS norm in the carrier comparison \\
$a_b(t),s_b,P_r$ & Multiplexed carriers, grouped bit signs, and restricted spatial projection \\
\bottomrule
\end{longtable}
}
\end{document}